\documentclass[a4paper,oribibl,envcountsame,runningheads]{llncs}

\def\orcid#1{{\href{http://orcid.org/#1}{\protect\raisebox{-1.25pt}{\protect\includegraphics{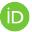}}}}}

\usepackage[latin1]{inputenc}
\usepackage[para]{footmisc}
\usepackage{amsmath}
\usepackage{amsfonts}
\usepackage{amssymb}
\usepackage{multicol}
\usepackage{multirow}
\usepackage[final]{graphicx}
\usepackage{algpseudocode,algorithm,algorithmicx}
\usepackage[disable]{todonotes} %
\usepackage{pifont}
\usepackage{mathpartir}
\usepackage{tikz}
\usetikzlibrary{shapes,arrows,positioning,matrix}
\usepackage{xspace}
\usepackage{csquotes}
\usepackage{enumerate}
\usepackage{cite}
\usepackage{thm-restate}
\usepackage[inline]{enumitem}
\usepackage{stmaryrd}

\usepackage[
   a4paper,
   pdftex,
   pdftitle={Connection-minimal Abduction in EL},
   pdfauthor={Fajar Haifani, Patrick Koopmann, Sophie Tourret, Christoph Weidenbach},
   pdfkeywords={},
   pdfborder={0 0 0},
   draft=false,
   bookmarksnumbered,
   bookmarks,
   bookmarksdepth=2,
   bookmarksopenlevel=2,
   bookmarksopen]{hyperref}

\DeclareFontFamily{OT1}{pzc}{}
\DeclareFontShape{OT1}{pzc}{m}{it}{<-> s * [1.10] pzcmi7t}{}
\DeclareMathAlphabet{\mathcalx}{OT1}{pzc}{m}{it}

\let\oldlabelitemi=\labelitemi
\let\labelitemi=\labelitemii %
\let\labelitemii=\oldlabelitemi %

\def\negvthinspace{\kern-0.083333em}
\def\vthinspace{\kern+0.083333em}
\def\vvthinspace{\kern+0.0416667em}
\def\negvvthinspace{\kern-0.0416667em}

\let\oldSigma=\Sigma
\renewcommand\Sigma{\mathrm{\oldSigma}}

\newcommand{\ELK}{\textsc{Elk}\xspace}

\newcommand{\Amc}{\ensuremath{\mathcal{A}}\xspace}
\newcommand{\Bmc}{\ensuremath{\mathcal{B}}\xspace}

\newcommand{\Hmc}{\ensuremath{\mathcal{H}}\xspace}
\newcommand{\Imc}{\ensuremath{\mathcal{I}}\xspace}
\newcommand{\Jmc}{\ensuremath{\mathcal{I}}\xspace}

\newcommand{\Mmc}{\ensuremath{\mathcal{M}}\xspace}

\newcommand{\Smc}{\ensuremath{\mathcal{S}}\xspace}
\newcommand{\Tmc}{\ensuremath{\mathcal{T}}\xspace}
\newcommand{\Rmc}{\ensuremath{\mathcal{R}}\xspace}

\newcommand{\EL}{\ensuremath{\mathcal{E\hspace{-.1em}L}}\xspace}
\newcommand{\ALC}{\ensuremath{\mathcal{A\hspace{-.1em}L\hspace{-.1em}C}}\xspace}
\newcommand{\NC}{\ensuremath{\mathsf{N_C}}\xspace}
\newcommand{\NR}{\ensuremath{\mathsf{N_R}}\xspace}
\newcommand{\NS}{\ensuremath{\mathsf{N_S}}\xspace}
\newcommand{\isa}{\sqsubseteq}
\newcommand{\dlAnd}{\sqcap}
\newcommand{\smand}{\ensuremath{\preceq_\dlAnd}}
\makeatletter
\newcommand*{\@updown}[2]{%
	\rotatebox[origin=c]{180}{$\m@th#1#2$}%
}
\providecommand{\bigsqcap}{%
	\mathop{%
		\mathpalette\@updown\bigsqcup
	}%
}
\makeatother

\newcommand{\CM}{\ensuremath{\Mmc}\xspace}%
\newcommand{\CMA}{\ensuremath{{\CM_A}}\xspace}%
\newcommand{\CMr}{\ensuremath{{\CM_r}}\xspace}%

\newcommand{\PIp}{\ensuremath{{\mathcal{PI}^{g+}_\Sigma(\Phi)}}}
\newcommand{\PIn}{\ensuremath{\mathcal{PI}^{g-}_\Sigma(\Phi)}}

\newcommand{\sig}{\ensuremath{\Sigma}\xspace}

\newcommand{\Tr}{\ensuremath{\mathfrak{T}}\xspace}

\newcommand{\V}{\ensuremath{\mathcal{V}}\xspace}

\newcommand{\gterms}[1]{\ensuremath{\mathsf{T}_{#1}(\NS)}}
\newcommand{\sk}{\ensuremath{\mathtt{sk}}}
\newcommand{\skz}{\ensuremath{\sk_0}}

\newcommand{\negc}[1]{#1^-}

\newcommand{\tup}[1]{(#1)}

\newcommand{\s}[1]{\ensuremath{\mathsf{#1}}\xspace}
\newcommand{\emp}{\s{employment}}
\newcommand{\RPo}{\s{ResearchPosition}}
\newcommand{\qual}{\s{qualification}}
\newcommand{\Dip}{\s{Diploma}}
\newcommand{\Res}{\s{Researcher}}
\newcommand{\wri}{\s{writes}}
\newcommand{\RPa}{\s{ResearchPaper}}
\newcommand{\Doc}{\s{Doctor}}
\newcommand{\PhD}{\s{PhD}}
\newcommand{\Prf}{\s{Professor}}
\newcommand{\FPr}{\s{FundsProvider}}
\newcommand{\GAp}{\s{GrantApplication}}
\newcommand{\Chr}{\s{Chair}}

\newcommand{\Obs}{\ensuremath{\alpha}\xspace} %
\newcommand{\Obsa}{\ensuremath{\alpha_{\text{a}}}\xspace} %
\newcommand{\Tac}{\ensuremath{\Tmc_{\text{a}}}\xspace}
\newcommand{\Hac}{\ensuremath{\Hmc_{\text{a}1}}\xspace}
\newcommand{\nHac}{\ensuremath{\Hmc_{\text{a}2}}\xspace}

\newcommand{\TwoExpTime}{\textsc{2ExpTime}\xspace}

\newcommand{\forest}{\ensuremath{\mathfrak{S}}\xspace}%
\newcommand{\sklab}{\ensuremath{\mathfrak{s}}}
\newcommand{\plab}{\ensuremath{{\mathfrak{l}^+}}}
\newcommand{\nlab}{\ensuremath{{\mathfrak{l}^-}}}
\newcommand{\impfrmtr}{\ensuremath{\varphi_{\forest,\nlab}}\xspace}
\newcommand{\impfrmtrp}{\ensuremath{\varphi_{\forest',\nlab'}}\xspace}

\newcommand{\qedhere}{\qed}

\begin{document}

\title{Connection-minimal Abduction in \EL\ via Translation to FOL -- Technical Report}
\titlerunning{Abduction in \EL\ via FOL}
\author{Fajar Haifani\inst{1,2}\orcid{0000-0001-5139-4503}
  \and
  Patrick Koopmann\inst{3}\orcid{0000-0001-5999-2583}
  \and
  Sophie Tourret\inst{4,1}\orcid{0000-0002-6070-796X}
  \and
  Christoph Weidenbach\inst{1}\orcid{0000-0001-6002-0458}
  }
\authorrunning{F.~Haifani et al.}
\institute{
  Max-Planck-Institut f\"ur Informatik, Saarland Informatics Campus,
  Saarbr\"ucken, Germany\\
  \and
  Graduate School of Computer Science, Saarbr\"ucken, Germany
  \email{\{f.haifani,c.weidenbach\}@mpi-inf.mpg.de}
  \and
  TU Dresden, Dresden, Germany\\
  \email{patrick.koopmann@tu-dresden.de}
\and
Universit\'e de Lorraine, CNRS, Inria, LORIA, Nancy, France\\
\email{sophie.tourret@inria.fr}
}
\maketitle      

\begin{abstract}
  Abduction in description logics finds extensions of a knowledge base
  to make it entail an observation. As such, it can be used
  to explain why the observation does not follow, to repair incomplete
  knowledge bases, and to provide possible explanations for unexpected
  observations. We consider TBox abduction in the lightweight description logic
  \EL, where the observation is a concept inclusion and the background knowledge
  is a TBox, i.e., a set of concept inclusions.
  To avoid useless answers, such problems usually come with further restrictions on
  the solution space and/or minimality criteria that help
  sort the chaff from the grain. We argue that existing minimality notions
  are insufficient, and %
  introduce connection minimality. %
  This criterion follows Occam's razor by rejecting hypotheses that use concept inclusions unrelated to the problem at hand.
  We show how to compute a special %
  class of connection-minimal hypotheses in a
  sound and complete way.
  Our technique is based on a translation to first-order logic, and constructs hypotheses
  based on prime implicates.
  We evaluate a prototype implementation of our approach on ontologies from the medical
  domain.
\end{abstract}

\section{Introduction}
\label{sec:intro}

Ontologies are used in areas like biomedicine or the semantic web to represent and reason about terminological knowledge.
They consist normally of a set of axioms formulated in a description logic (DL), giving definitions of concepts, or stating relations between them.
In the lightweight description logic \EL\ \cite{DL_TEXTBOOK}, particularly used in the biomedical domain, we find ontologies that contain around a hundred thousand axioms. For instance, SNOMED CT\footnote{\url{https://www.snomed.org/}} contains over 350,000 axioms, and the Gene Ontology GO\footnote{\url{http://geneontology.org/}} defines over 50,000 concepts. 
A central reasoning task for ontologies is to determine whether one concept is subsumed by another, a question that can be answered in polynomial time \cite{DBLP:conf/ijcai/BaaderBL05}, and rather efficiently in practice using highly optimized description logic reasoners~\cite{REASONER_PERFORMANCE}. If the answer to this question is unexpected or hints at an error, a natural interest is in an explanation for that 
answer---especially if the ontology is complex.
But whereas explaining entailments---i.e., explaining why a concept subsumption
holds---is well-researched in the DL literature %
and integrated into standard ontology editors~\cite{JUSTIFICATIONS_PROTEGE,PROOFS_PROTEGE},
the problem of explaining non-entailments has received less attention, and there is no standard tool support.
Classical approaches involve counter-examples~\cite{EXPLAINING_BY_EXAMPLE}, or \emph{abduction}.

In abduction a non-entailment $\Tmc\not\models\alpha$, for a TBox \Tmc and an observation $\alpha$, is explained by providing a \enquote{missing piece}, the \emph{hypothesis}, that, when added to the ontology, would entail $\alpha$.
Thus %
it provides possible fixes in case the entailment should hold. In the DL context, depending on the shape of the observation, one distinguishes between concept abduction~\cite{CONCEPT_ABDUCTION}, ABox abduction~\cite{OUR_IJCAI_ABDUCTION,WARREN_ABOX_ABDUCTION,EX_RULES_ABDUCTION,%
AAA-ABDUCTION-SOLVER,PukancovaHomola2017,%
DL_LITE_ABDUCTION,DuWangShen2014,DuQiShenPan2012,HallandBritzABox2012,KlarmanABox2011%
}, TBox abduction~\cite{DU_TBOX_ABDUCTION,WEI_TBOX_ABDUCTION} or knowledge base
abduction~\cite{OUR_KR_ABDUCTION,FOUNDATIONS_DL_ABDUCTION}. We are focusing here
on TBox abduction, where the ontology and hypothesis are TBoxes and the
observation is a concept inclusion (CI), i.e., a single TBox axiom.

To illustrate this problem,
consider the following TBox, about academia,
\begin{eqnarray*}
	\Tac&=\{&\exists\emp.\RPo\dlAnd \exists \qual.\Dip \isa \Res,\\
      &&\exists \wri.\RPa \isa \Res,\,\Doc\isa \exists \qual.\PhD,\\ 
      &&\Prf\equiv \Doc\dlAnd \exists\emp.\Chr,\\
      &&\FPr\isa \exists \wri.\GAp\,\}
\end{eqnarray*}
that states, in natural language:
\begin{itemize}
\item ``Being employed in a research position and having a qualifying diploma implies being a researcher.''
\item ``Writing a research paper implies being a researcher.''
\item ``Being a doctor implies holding a PhD qualification.''
\item ``Being a professor is being a doctor employed at a (university) chair.''
\item ``Being a funds provider implies writing grant applications.''
\end{itemize}
The observation $\Obsa=\Prf\sqsubseteq \Res$, ``Being a professor implies being a researcher'', does not follow from $\Tac$ although it should. We can use TBox abduction to find different ways of recovering this entailment. %

Commonly,  to avoid trivial answers, the user provides syntactic restrictions on hypotheses, such as a set of abducible axioms to pick from~\cite{EX_RULES_ABDUCTION,PukancovaHomola2017}, a set of abducible predicates~\cite{OUR_KR_ABDUCTION,OUR_IJCAI_ABDUCTION}, or patterns on the shape of the solution~\cite{du2017practical}.
But even with those restrictions in place, there may be many possible solutions and, to find the ones with the best explanatory potential, syntactic criteria are usually
combined with minimality criteria
such as subset minimality, size minimality, or semantic minimality~\cite{DL_LITE_ABDUCTION}.
Even combined, these minimality criteria still retain a major flaw.
They allow for explanations that go against the principle of parsimony, also known as Occam's razor, %
in that they may contain concepts that are completely unrelated to the problem at hands.
As an illustration, let us return to our academia example.
The TBoxes
\begin{eqnarray*}
\Hac &=\{&\Chr\isa \RPo,\, \PhD\isa \Dip\}\text{ and}\\
\nHac &=\{&\Prf\isa \FPr,\,\GAp\isa \RPa\}
\end{eqnarray*}
are two hypotheses solving the TBox abduction problem involving $\Tac$ and $\Obsa$.
Both of them are subset-minimal, have the same size, and are incomparable w.r.t.\ the entailment relation, so that traditional minimality criteria cannot distinguish them.
However, intuitively, the second hypothesis feels more arbitrary than the first.
Looking at $\Hac$, $\Chr$ and $\RPo$ occur in \Tac in concept inclusions where the concepts in \Obsa also occur, and both $\PhD$ and $\Dip$ are similarly related to \Obsa but via the role $\qual$.
In contrast, $\nHac$ involves the concepts $\FPr$ and $\GAp$ that are not related to \Obsa in any way in \Tac. 
In fact, any random concept inclusion $A\isa \exists \s{writes}. B$ in $\Tac$ would lead to a hypothesis similar to $\nHac$ where $A$ replaces $\FPr$ and $B$ replaces $\GAp$.
Such explanations are not parsimonious.

We introduce a new minimality criterion called \emph{connection minimality} that is parsimonious (Sect.\ \ref{sec:conmin}), defined for the lightweight description logic \EL.
This criterion characterizes hypotheses for $\Tmc$ and $\Obs$ that connect the left- and right-hand sides of the observation $\Obs$ without introducing spurious connections.
To achieve this, every left-hand side of a CI in the hypothesis must follow from the left-hand side of $\Obs$ in $\Tmc$, and, taken together, all the right-hand sides of the CIs in the hypothesis must imply the right-hand side of $\Obs$ in $\Tmc$, as is the case for $\Hac$.
To compute connection-minimal hypotheses in practice, we present a technique based on first-order reasoning that
proceeds in three steps (Sect.\ \ref{sec:technique}).
First, we translate the abduction problem into a first-order formula~$\Phi$.
We then compute the prime implicates of $\Phi$, that is, a set of minimal logical consequences of $\Phi$ that subsume all other consequences of $\Phi$.
In the final step, we
construct, based on those prime implicates, solutions to the original problem.
We prove that all hypotheses generated in this way satisfy the connection minimality criterion, and that the  method is complete for a relevant
subclass of connection-minimal hypotheses.
We use the %
SPASS theorem prover \cite{DBLP:conf/cade/WeidenbachSHRT07} as a restricted  SOS-resolution~\cite{WosRobinsonEtAl65,DBLP:conf/cade/HaifaniTW21} engine
for the computation of prime implicates in a prototype implementation (Sect.\ \ref{sec:implem}), and we present an experimental analysis of its performances on a set of bio-medical ontologies.(Sect.\ \ref{sec:exp}).
Our results indicate that our method can in many cases be applied in practice to compute connection-minimal hypotheses.

There are not many techniques that can handle TBox abduction in \EL or more expressive DLs \cite{DU_TBOX_ABDUCTION, WEI_TBOX_ABDUCTION, OUR_KR_ABDUCTION}.
In~\cite{DU_TBOX_ABDUCTION}, instead of a set of abducibles, a set of \emph{justification patterns} is given,
in which the solutions have to fit.
An arbitrary oracle function is used to decide whether a solution is admissible or not (which may use abducibles, justification patterns, or something else), and it is shown that deciding the existence of hypotheses is tractable. However, different to our approach, they only consider atomic CIs in hypotheses, while we also allow for
hypotheses involving conjunction. %
The setting from \cite{WEI_TBOX_ABDUCTION} also considers \EL, and abduction under various minimality notions such as subset minimality and size minimality. 
It presents practical algorithms, and an evaluation of an implementation for an always-true informativeness oracle (i.e., limited to subset minimality).
Different to our approach, it uses an external DL reasoner to decide entailment relationships.
In contrast, we present an approach that directly exploits first-order reasoning, and thus has the potential to be generalisable to more expressive DLs.

While dedicated resolution calculi have been used before to solve abduction in DLs~\cite{OUR_KR_ABDUCTION,WARREN_ABOX_ABDUCTION}, to the best of our knowledge,
the only work that relies on first-order reasoning for DL abduction is \cite{KlarmanABox2011}.
Similar to our approach, it uses SOS-resolution, but to perform ABox adbuction for the
more expressive DL \ALC.
Apart from the different problem solved, in contrast to~\cite{KlarmanABox2011} we also provide a semantic characterization of the hypotheses generated by our method.
We believe this characterization to be a major contribution of our paper.
It provides an intuition of what parsimony is for this problem, independently of one's ease with first-order logic calculi, which should facilitate the adoption of this minimality criterion by the DL community.
Thanks to this characterization, our technique is calculus agnostic.
Any method to compute prime implicates in first-order logic can be a basis for our abduction technique,
without additional theoretical work, which is not the case for \cite{KlarmanABox2011}. %
Thus, abduction in \EL can benefit from the latest advances in prime implicates generation in first-order logic.

\section{Preliminaries}
\label{seq:prelim}
    We first recall the descripton logic \EL and its translation to first-order logic~\cite{DL_TEXTBOOK}, as well as
	TBox abduction in this logic.

	Let $\NC$ and $\NR$ be pair-wise disjoint, countably infinite sets of unary predicates called \emph{atomic concepts} and of binary predicates called \emph{roles}, respectively.
	Generally, we use letters $A$, $B$, $E$, $F$,...\ for atomic concepts, and $r$ for roles, possibly annotated.
	Letters $C$, $D$, possibly annotated, denote \emph{\EL concepts}, built according to the syntax rule
	$$
	C \ ::= \ \top \ \mid \  
	A \ \mid \ C\sqcap C \ \mid \ \exists r.C  \ .
	$$
	We implicitly represent $\EL$ conjunctions as sets, that is, without order, nested conjunctions, and multiple occurrences of a conjunct.
	We use $\bigsqcap\{C_1,\ldots,C_m\}$ to abbreviate $C_1\sqcap\ldots\sqcap C_m$, and identify the
	empty conjunction ($m=0$) with $\top$.
	An \emph{\EL TBox} $\Tmc$ is a finite set of \emph{concept inclusions} (CIs) of the form $C\isa D$.

    \EL is a syntactic variant of a fragment of first-order logic that uses $\NC$ and $\NR$ as predicates.
    Specifically, TBoxes $\Tmc$ and CIs $\alpha$ correspond to closed
    first-order formulas $\pi(\Tmc)$ and $\pi(\alpha)$ resp., while concepts $C$ correspond to open
	formulas $\pi(C,x)$ with a free variable~$x$. In particular, we have
		\begin{align*}
			\pi(\top,x)&:=\textbf{true},
			& \qquad
			\pi(\exists r. C,x)&:=\exists y.(r(x,y)\land\pi (C,y)),
			\\
			\pi(A,x)&:=A(x),
			& \qquad
			\pi(C\sqsubseteq D)&:=\forall x.(\pi(C,x)\rightarrow \pi(D,x)),
			\\
			\pi(C\dlAnd D,x)&:=\pi(C,x)\land \pi(D,x),
			& \quad
			\pi(\Tmc)&:=\bigwedge\{\pi(\alpha)\mid \alpha\in\Tmc\}.
		\end{align*}
    As common, we often omit the $\bigwedge$ in conjunctions $\bigwedge \Phi$, that is, we identify sets of
    formulas with the conjunction over those.
    The notions of a \emph{term} $t$; an \emph{atom} $P(\bar{t})$ where $\bar{t}$ is a sequence of terms; a \emph{positive literal} $P(\bar{t})$;
    a \emph{negative literal} $\neg P(\bar{t})$; and a clause, Horn, definite, positive or negative, are defined as usual for first-order logic, and so are
    entailment and satisfaction of first-order formulas.

    We identify CIs and TBoxes with their translation into first-order logic,
    and can thus speak of the entailment between formulas, CIs and TBoxes.
    When $\Tmc\models C\isa D$ for some \Tmc, we call $C$ a \emph{subsumee} of $D$ and $D$ a \emph{subsumer} of $C$.
    We adhere here to the definition of the word ``subsume'': ``to include or contain something else'', although the terminology is reversed in first-order logic.
    We say two TBoxes $\Tmc_1$, $\Tmc_2$ are \emph{equivalent}, denoted $\Tmc_1\equiv\Tmc_2$ iff $\Tmc_1\models\Tmc_2$ and $\Tmc_2\models\Tmc_1$.
    For example $\{D\isa C_1,\ldots, D\isa C_n\}\equiv \{D\isa C_1\sqcap\ldots\sqcap C_n\}$.
    It is well known that, due to the absence of concept negation, every \EL TBox is consistent.

	The abduction problem we are concerned with in this paper is the following:
	\begin{definition}
    An \emph{\EL TBox abduction problem} (shortened to \emph{abduction problem}) is a tuple $\langle\Tmc,\sig,C_1\isa C_2\rangle$, where
    $\Tmc$ is a TBox called the \emph{background knowledge}, $\sig$ is a set of atomic concepts called the
    \emph{abducible signature}, and $C_1\sqsubseteq C_2$ is a CI called the \emph{observation}, s.t.\ $\Tmc\not\models C_1\isa C_2$.
    A solution to this problem is a TBox
    \[
      \Hmc\subseteq\left\{A_{1}\sqcap\dots\sqcap A_{n}\isa B_{1}\sqcap\dots\sqcap B_{m} \mid
      \{A_{1},\dots, A_{n},B_{1},\dots, B_{m}\}\subseteq \sig\right\}
    \]
    where $m>0$, $n\geq 0$
    and such that $\Tmc\cup\Hmc\models C_1\isa C_2$ and, for all CIs $\alpha\in\Hmc$, $\Tmc\not\models\alpha$.
    A solution to an abduction problem is called a \emph{hypothesis}.
  \end{definition}
  For example, $\Hac$ and $\nHac$ are solutions for $\langle \Tac,\Sigma,\Obsa \rangle$, as long as $\Sigma$ contains all the atomic concepts that occur in them.
  Note that in our setting, as in~\cite{CONCEPT_ABDUCTION,WEI_TBOX_ABDUCTION},
  concept inclusions in a hypothesis are \emph{flat}, i.e., they contain no existential role restrictions.
  While this restricts the solution space for a given problem, it is possible to bypass this limitation in a targeted way, by introducing fresh atomic concepts equivalent to a concept of interest.
  We exclude the consistency requirement $\Tmc\cup\Hmc\not\models\bot$, that is given in other definitions of DL abduction problem \cite{OUR_IJCAI_ABDUCTION}, since \EL TBoxes are always consistent.
  We also allow $m>1$ instead of the usual $m=1$.
  This produces the same hypotheses modulo equivalence. %

  For simplicity, we assume in the following that the concepts $C_1$ and $C_2$ in the abduction problem are atomic.
  We can always introduce fresh atomic concepts $A_1$ and $A_2$ with $A_1\sqsubseteq C_1$ and $C_2\sqsubseteq A_2$ to solve the problem for complex concepts.

  Common minimality criteria include \emph{subset} minimality, \emph{size} minimality and \emph{semantic} minimality, that respectively favor $\Hmc$ over $\Hmc'$ if: $\Hmc\subsetneq\Hmc'$; the number of atomic concepts in $\Hmc$ is smaller than in $\Hmc'$; and if $\Hmc\models\Hmc'$ but $\Hmc'\not\models\Hmc$.

\section{Connection-minimal Abduction}
\label{sec:conmin}

To address the lack of parsimony of common minimality criteria, illustrated in the academia example, we introduce
\emph{connection} minimality,
Intuitively, connection minimality only accepts those hypotheses that ensure that every CI in the hypothesis is
connected to both $C_1$ and $C_2$ in $\Tmc$,
as is the case for $\Hac$ in the academia example. %
The definition of connection minimality is based on the following ideas: 
\begin{enumerate*}[label=\arabic*)]
   \item \label{e:d} Hypotheses for the abduction problem %
     should create a \emph{connection} between $C_1$ and $C_2$, which can be seen as a concept $D$ that satisfies
     $\Tmc\cup\Hmc\models C_1\sqsubseteq D$, $D\sqsubseteq C_2$.
   \item \label{e:hom} To ensure parsimony, we want this connection to be based on concepts $D_1$ and $D_2$ for which we already have $\Tmc\models C_1\sqsubseteq D_1$, $D_2\sqsubseteq C_2$.
     This prevents the introduction of unrelated concepts in the hypothesis.
     Note however that $D_1$ and $D_2$ can be complex, thus the connection from $C_1$ to $D_1$ (resp.\ $D_2$ to $C_2$) can be established by arbitrarily long chains of concept inclusions.
   \item \label{e:smand} We additionally want to make sure that the connecting concepts are not more complex than necessary, and that $\Hmc$ only contains CIs that directly connect parts of $D_2$ to parts of $D_1$ by closely following their structure.
   \end{enumerate*}

   To address point\ \ref{e:d}, we simply introduce connecting concepts formally.
   \begin{definition}
     Let $C_1$ and $C_2$ be concepts. %
     A concept $D$ \emph{connects} $C_1$ to $C_2$ in $\Tmc$ if and only if $\Tmc\models C_1\isa D$ and $\Tmc\models D\isa C_2$.
   \end{definition}
   Note that if $\Tmc \models C_1 \isa C_2$ then both $C_1$ and $C_2$ are connecting concepts from $C_1$ to $C_2$, and if $\Tmc \not\models C_1 \isa C_2$, the case of interest, neither of them are. %

	To address point\ \ref{e:hom}, we must capture \emph{how} a hypothesis creates the connection between the
	concepts $C_1$ and $C_2$. As argued above, this is established via concepts $D_1$ and $D_2$ that
	satisfy $\Tmc\models C_1\sqsubseteq D_1$, $D_2\sqsubseteq C_2$.
  Note that having only two concepts $D_1$ and $D_2$ is exactly what makes the approach parsimonious.
  If there was only one concept, $C_1$ and $C_2$ would already be connected, and as soon as there are more than two concepts,
  hypotheses start becoming more arbitrary:
  for a very simple example with unrelated concepts, assume given a TBox that entails $\s{Lion}\isa\s{Felidae}$, $\s{Mammal}\isa\s{Animal}$ and $\s{House}\isa\s{Building}$.
  A possible hypothesis to explain $\s{Lion}\isa\s{Animal}$ is $\{\s{Felidae}\isa\s{House},\s{Building}\isa\s{Mammal}\}$ but this explanation is more arbitrary than $\{\s{Felidae}\isa\s{Mammal}\}$---as is the case when comparing $\nHac$ with $\Hac$ in the academia example---because of the lack of connection of $\s{House}\isa\s{Building}$ with both $\s{Lion}$ and $\s{Animal}$.
  Clearly this CI could be replaced by any other CI entailed by \Tmc, which is what we want to avoid.

  We can represent the structure of $D_1$ and $D_2$ in graphs by using \emph{\EL description trees}, originally from Baader et al.\ \cite{baader1999computing}.

	\begin{definition}
		An \emph{\EL description tree} is a finite labeled tree $\Tr=(V,E,v_0,l)$ where
     $V$ is a set of nodes with root $v_0\in V$,
    the nodes $v\in V$ are labeled with $l(v)\subseteq\NC$, and
    the (directed) edges $vrw \in E$ are such that $v,w\in V$ and are labeled with $r\in\NR$.
	\end{definition}
	Given a tree $\Tr=(V,E,v_0,l)$ and $v\in V$, we denote by $\Tr(v)$ the subtree of $\Tr$ that is rooted in $v$.
	If $l(v_0)=\{A_1,\ldots,A_k\}$ and $v_1$, $\ldots$, $v_n$ are all the children of $v_0$, we
	can define the concept represented by $\Tr$ recursively using
	\(
        C_\Tr=A_1\dlAnd \ldots \dlAnd A_k\dlAnd \exists r_1. C_{\Tr(v_1)}\dlAnd\ldots
        \dlAnd \exists r_l.C_{\Tr(v_l)}
	\)
    where for $j\in\{1,\ldots,n\}$, $v_0 r_j v_j\in E$. Conversely, we can define $\Tr_C$ for
    a concept $C=A_1\sqcap \ldots\sqcap A_k\sqcap\exists r_1.C_1\sqcap\ldots\sqcap\exists r_n.C_n$ inductively based on the
    pairwise disjoint description trees $\Tr_{C_i}=\{V_i, E_i, v_i, l_i\}$, $i\in\{1,\ldots, n\}$. Specifically,
     $\Tr_C=(V_C, E_C,v_C, l_C)$, where
     \begin{multicols}{2}
       \noindent
         $V_C=\{v_0\}\cup\bigcup_{i=1}^{n} V_i$,\\
         $E_C=\{v_0r_iv_i\mid 1\leq i\leq n\}\cup\bigcup_{i=1}^{n} E_i$,\\

         \columnbreak
         \noindent
         $l_C(v)=l_i(v)$ for $v\in V_i$,\\
         $l_C(v_0)=\{A_1,\ldots, A_k\}$.\\
    \end{multicols}

    If $\Tmc=\emptyset$, then subsumption between \EL concepts is characterized by the existence of a homomorphism
    between the corresponding description trees~\cite{baader1999computing}.
    We generalise this notion to also take the TBox into account.
	\begin{definition}
		\label{def:homomorph}
		Let $\Tr_1=(V_1,E_1,v_0,l_1)$ and $\Tr_2=(V_2,E_2,w_0,l_2)$ be two description trees and $\Tmc$ a
		TBox.
		A mapping $\phi: V_2\rightarrow V_1$ is a \emph{\Tmc-homomorphism}
		from $\Tr_2$ to $\Tr_1$
		if and only if the following conditions are satisfied:
		\begin{enumerate}
			\item\label{it:root}$\phi(w_0)=v_0$
			\item$\phi(v)r\phi(w)\in E_1$ for all $vrw\in E_2$
			\item\label{it:entails} for every $v\in V_1$ and $w\in V_2$ with
			$v=\phi(w)$, $\Tmc\models \bigsqcap l_1(v)\isa \bigsqcap l_2(w)$
		\end{enumerate}
		If only 1 and 2 are satisfied, then $\phi$ is called a \emph{weak} homomorphism.
	\end{definition}
	\Tmc-homomorphisms for a given TBox \Tmc capture subsumption w.r.t.~\Tmc.
  If there exists a \Tmc-homomorphism $\phi$ from $\Tr_2$ to $\Tr_1$,
	then $\Tmc\models C_{\Tr_1}\isa C_{\Tr_2}$.
	This can be shown easily by structural induction using the definitions (see App~\ref{app:homo-ent}).
  The weak homomorphism is the structure on which a \Tmc-homomorphism can be built by adding some hypothesis \Hmc\ to \Tmc.
  It is used to reveal missing links between a subsumee $D_2$ of $C_2$ and a subsumer $D_1$ of $C_1$, that can be added using \Hmc.

  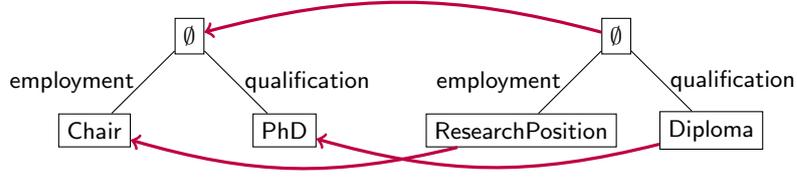
\begin{figure}[t]
    \centering
    \begin{tikzpicture}[level distance=1.25cm,
			level 1/.style={sibling distance=11cm},
			level 2/.style={sibling distance=2.5cm}]
			
			\node{}
			child{edge from parent[draw=none] node(v0)[draw] {$\emptyset$}
        child { node(v2)[draw] {{$\Chr$}}edge from parent node[left] {{$\emp$}}
        }
        child { node(v3)[draw] {{$\PhD$}}edge from parent node[right] {{$\qual$}}
        }
			}
			child{edge from parent[draw=none] node(w0)[draw] {$\emptyset$}
					child { node(w2)[draw] {$\RPo$}edge from parent node[left] {$\emp$}
					}
					child { node(w3)[draw] {$\Dip$}edge from parent node[right] {$\qual$}
					}
			};
			\draw[color=purple,->,very thick] (w0) to [bend right=15] (v0);
			\draw[color=purple,->,very thick] (w2) to [bend left=15] (v2);
			\draw[color=purple,->,very thick] (w3) to [bend left=15] (v3);
    \end{tikzpicture}
    \caption{Description trees of $D_1$ (left) and $D_2$ (right).}
    \label{fig:homomorph}
  \end{figure}

  \begin{example}
    \label{ex:trees}
    Consider the concepts
		\begin{align*}
      D_1&= \exists\emp.\Chr\dlAnd\exists\qual.\PhD\\
      D_2&=  \exists\emp.\RPo\dlAnd\exists\qual.\Dip
		\end{align*}
    from the academia example.
		Figure\ \ref{fig:homomorph} illustrates description trees for $D_1$ (left) and $D_2$ (right).
		The curved arrows show a weak homomorphism from $\Tr_{D_2}$ to $\Tr_{D_1}$ that can be strengthened into a \Tmc-homomorphism for some TBox $\Tmc$ that
		corresponds to the set of CIs in \text{$\Hac \cup \{\top\isa\top\}$}.
		The figure can also be used to illustrate what we mean by connection minimality: in order to create a connection
    between $D_1$ and $D_2$, we should \emph{only} add the CIs from $\Hac\cup \{\top\isa\top\}$ %
    \emph{unless} they are already entailed by $\Tac$.
    In practice, this means the weak homomorphism from $D_2$ to $D_1$ becomes a $(\Tac\cup\Hac)$-homomorphism.
  \end{example}

   To address point\ \ref{e:smand}, we define a partial order $\smand$ on concepts, s.t.\ $C\smand D$ if we can turn $D$ into $C$ by removing conjuncts in subexpressions, e.g., $\exists r'. B \smand \exists r. A \dlAnd \exists r'. (B \dlAnd B') $.
   Formally, this is achieved by the following definition.
	\begin{definition}
		\label{def:conmin}
		Let $C$ and $D$ be arbitrary concepts.
		Then $C\smand D$ if either:
		\begin{itemize}
			\item $C = D$,
			\item $D = D' \dlAnd D''$, and $C\smand D'$, or
			\item $C = \exists r.C'$, $D = \exists r.D'$ and $C'\smand D'$.
		\end{itemize}
	\end{definition}

    We can finally capture our ideas on connection minimality formally.

    \begin{definition}[Connection-Minimal Abduction]
      \label{def:abd}
      Given an abduction problem $\langle\Tmc,\sig,C_1\isa C_2\rangle$, a hypothesis \Hmc is \emph{connection-minimal} if there exist concepts $D_1$ and $D_2$ built over $\Sigma\cup\NR$ and a mapping $\phi$ satisfying each of the following conditions:
      \begin{enumerate}
			\item \label{it:reach} $\Tmc\models C_1\isa D_1$,
			\item \label{it:reachbw} $D_2$ is a \smand-minimal concept s.t.\ $\Tmc\models D_2\isa C_2$,
      \item \label{it:hom} $\phi$ is a weak homomorphism from the tree $\Tr_{D_2}=(V_2,E_2,w_0,l_2)$ to the tree $\Tr_{D_1}=(V_1,E_1,v_0,l_1)$, and
      \item \label{it:H} $\Hmc=\{\bigsqcap l_1(\phi(w))\isa \bigsqcap l_2(w)\mid w\in V_2\wedge\Tmc\not\models\bigsqcap l_1(\phi(w))\isa \bigsqcap l_2(w)\}$.
      \end{enumerate} 
      $\Hmc$ is additionally called \emph{packed} if
      the left-hand sides of the CIs in $\Hmc$ cannot hold more conjuncts than they do, which is formally stated as:
      for \Hmc, there is no \Hmc' defined from the same $D_2$ and a $D_1'$ and $\phi'$ s.t.\ there is a node $w\in V_2$ for which $l_1(\phi(w))\subsetneq l_1'(\phi'(w))$ and $l _1(\phi(w'))=l_1'(\phi'(w'))$ for $w'\neq w$.
     \end{definition}
Straightforward consequences of Def.\ \ref{def:abd} include that $\phi$ is a $(\Tmc\cup\Hmc)$-homomorphism from $\Tr_{D_2}$ to $\Tr_{D_1}$ and that $D_1$ and $D_2$ are connecting concepts from $C_1$ to $C_2$ in $\Tmc\cup\Hmc$ so that $\Tmc\cup\Hmc\models C_1\isa C_2$ as wanted (more details on theses results are given in App~\ref{app:straight_cons}). %
With the help of Fig.\ \ref{fig:homomorph} and Ex.~\ref{ex:trees}, one easily establishes that hypothesis $\Hac$ is connection-minimal---and even packed.
Connection-minimality rejects $\nHac$, as a single $\Tmc'$-homomorphism for some $\Tmc'$ between two concepts
$D_1$ and $D_2$ would be insufficient: we would need two weak homomorphisms, one linking $\Prf$ to $\FPr$ and another linking $\exists\wri.\GAp$ to $\exists\wri.\RPa$.

\section{Computing Connection-minimal Hypotheses using Prime Implicates}
\label{sec:technique}

To compute connection-minimal hypotheses in practice, we propose a method based on
first-order prime implicates, that can be derived by resolution. We assume the reader is familiar with the basics of
first-order resolution, and do not reintroduce notions of clauses, Skolemization
and resolution inferences here (for details, see~\cite{DBLP:books/el/RV01/BachmairG01}).
In our context, every term is built on variables, denoted $x$, $y$, a single constant $\skz$ and unary Skolem functions usually denoted $\sk$, possibly annotated. Prime implicates are defined as follows.

\begin{definition}[Prime Implicate]
  Let $\Phi$ be a set of clauses. A clause $\varphi$ is an \emph{implicate} of $\Phi$ if $\Phi\models \varphi$.
  Moreover $\varphi$ is \emph{prime} if for any other implicate $\varphi'$ of $\Phi$ s.t.\ $\varphi'\models \varphi$, it also holds that $\varphi\models\varphi'$.
\end{definition} 
Let $\sig\subseteq\NC$ be a set of unary predicates. Then $\PIp$ denotes the set of all positive ground prime implicates of $\Phi$ that only use predicate symbols from $\sig\cup\NR$,  while $\PIn$ denotes the set of all negative ground prime implicates of $\Phi$ that only use predicates symbols from $\sig\cup\NR$.

\begin{example}
  Given a set of clauses $\Phi= \{A_1(\skz),\neg B_1(\skz), \neg A_1(x)\lor r(x,\sk(x)),$\\
  $\neg A_1(x)\lor A_2(\sk(x)), \neg B_2(x)\lor \neg r(x,y)\lor \neg B_3(y)\lor B_1(x)\}$,
  the ground prime implicates of $\Phi$ for $\Sigma = \NC$ are, on the positive side, $\PIp=\{A_1(\skz),$ $A_2(\sk(\skz)), r(\skz,\sk(\skz))\}$ and, on the negative side, $\PIn=\{\neg B_1(\skz),$ $\neg B_2(\skz)\lor \neg B_3(\sk(\skz))\}$.
  They are implicates because all of them are entailed by $\Phi$.
  For a ground implicate $\varphi$, another ground implicate $\varphi'$ such that $\varphi'\models\varphi$ and $\varphi\not\models\varphi'$ can only be obtained from $\varphi$ by dropping literals.
  Such an operation does not produce another implicate for any of the clauses presented above as belonging to \PIp and \PIn, thus they really are all prime.
\end{example}

  To generate hypotheses, we translate the abduction problem into a set of first-order clauses, from which we can infer prime implicates that we then combine to obtain the result as illustrated in Fig.\ \ref{fig:birdeye}. %
\begin{figure}[t]
  \tikzstyle{block} = [rectangle, draw, fill=blue!20,
  text centered, minimum height=2em, minimum width=5em]
  \tikzstyle{centerblock} = [rectangle, draw, fill=green!20,
  text centered, minimum height=2em, minimum width=5em]
  \tikzstyle{line} = [draw, very thick, color=black!50, -latex']
  \tikzstyle{io} =    [draw, trapezium, trapezium left angle=70, trapezium right angle=-70,fill=red!20,
  minimum height=2em]
  \tikzstyle{iofol} =    [draw, trapezium, trapezium left angle=70, trapezium right angle=-70,fill=green!20,
  minimum height=2em]
  \center
  \begin{tikzpicture}[node distance = 0.5cm,every node/.style={scale=0.8},auto]
    \node [io] (obs) {$C_1\sqsubseteq C_2$};
    \node [io, below =of obs] (ontology) {$\Tmc$};

    \node [block, right =of ontology] (pi) {translation};
    \node [iofol, right =of pi] (phi) {$\Phi$};
    \node [block, right =of phi] (B) {PI generation};
    \node [io, below =of phi] (h) {$\Sigma$};
    \node [iofol, above right =of B] (piplus) {\PIp};
    \node [iofol, below right =of B] (piminus) {\PIn};
    \node [block, below right =of piplus] (comb) {recombination};  
    \node [io, right =of comb] (res) {\Smc};

    \path [line] (obs) -- (pi);
    \path [line] (ontology) -- (pi);
    \path [line] (h) -- (B);
    \path [line] (pi) -- (phi);
    \path [line] (phi) -- (B);
    \path [line] (B) -- (piplus);
    \path [line] (B) -- (piminus);
    \path [line] (piplus) -- (comb);
    \path [line] (piminus) -- (comb);
    \path [line] (comb) -- (res);
  \end{tikzpicture}
  \caption{\EL\ abduction using prime implicate generation in FOL.}
  \label{fig:birdeye}
\end{figure}
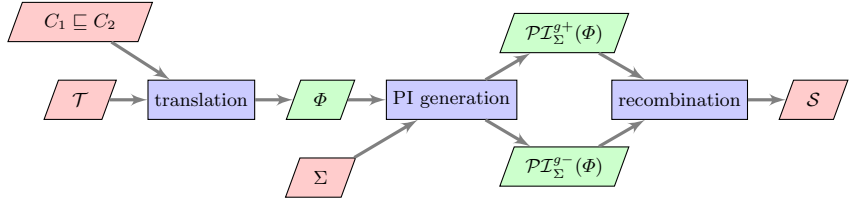
In more details: We first translate the problem into a set $\Phi$ of Horn clauses.
Prime implicates can be computed using an off-the-shelf tool\ \cite{DBLP:journals/aicom/NabeshimaIIR10,DBLP:conf/cade/EchenimPS18} or, in our case, a slight
extension of the resolution-based version of the SPASS theorem prover\ \cite{DBLP:conf/cade/WeidenbachSHRT07} using the set-of-support strategy and some added features described in Sect.\ \ref{sec:implem}.
Since $\Phi$ is Horn, \PIp\ contains only unit clauses.
A final recombination step looks at the clauses in \PIn\ one after the other.
These correspond to candidates for the connecting concepts $D_2$ of Def.~\ref{def:abd}.
Recombination attempts to match each literal in one such clause with unit clauses from \PIp.
If such a match is possible, it produces a suitable $D_1$ to match $D_2$, and allows the creation of
a solution to the abduction problem.
The set \Smc\ contains all the hypotheses thus obtained. %
In what follows, 
we present our translation of abduction problems into first-order logic and formalize the construction of hypotheses from the prime implicates of this translation. %
We then show how to obtain termination for the prime implicate generation process with soundness and completeness guarantees on the solutions computed.
\paragraph{\textbf{Abduction Method.}}

	We assume the \EL TBox in the input is in normal form as defined, e.g., by Baader et al.\ \cite{DL_TEXTBOOK}.
  Thus every CI is of one of the following forms:
  \[
    A \sqsubseteq B \qquad  A_1\sqcap A_2 \sqsubseteq B \qquad \exists r.A\sqsubseteq B \qquad  A \sqsubseteq\exists r.B
  \]
    where $A$, $A_1$, $A_2$, $B\in\NC\cup\{\top\}$.

    The use of normalization is justified by the following lemma (see App~\ref{app:preliminaries} for its proof).
    \begin{restatable}{lemma}{LemNormalize}\label{lem:normalize}
     For every \EL TBox $\Tmc$, we can compute in polynomial time an \EL TBox
      $\Tmc'$ in normal form such that for every other TBox $\Hmc$ and
every CI $C\sqsubseteq D$ that use only names occurring in $\Tmc$, we have
$\Tmc\cup\Hmc\models C\sqsubseteq D$ iff $\Tmc'\cup\Hmc\models C\sqsubseteq D$.
    \end{restatable}
  After the normalisation, we eliminate occurrences of $\top$, replacing this concept everywhere by the fresh atomic concept $A_\top$.
  We furthermore add $\exists r.A_\top\isa A_\top$ and $B\isa A_\top$ in $\Tmc$ for every role $r$ and atomic concept $B$ occurring in $\Tmc$.
	This simulates the semantics of $\top$ for $A_\top$, namely the implicit property that $C\isa \top$ holds for any $C$ no matter what the TBox is.
	In particular, this ensures that whenever there is a positive prime implicate $B(t)$ or $r(t,t')$, $A_\top(t)$ also becomes a prime implicate.
    Note that normalisation and $\top$ elimination extend the signature, and thus potentially the solution space of the abduction problem.
    This is remedied by intersecting the set of abducible predicates \sig with the signature of the original input ontology.
    We assume that $\Tmc$ is in normal form and without $\top$ in the rest of the paper.

	We denote by $\negc{\Tmc}$ the result of renaming all atomic concepts $A$ in $\Tmc$ using fresh \emph{duplicate} symbols $\negc{A}$.
  This renaming is done only on concepts but not on roles, and on $C_2$ but not on $C_1$ in the observation.
  This ensures that
  the literals in a clause of $\PIn$ all relate to the conjuncts of a $\smand$-minimal subsumee of $C_2$.
  Without it, some of these conjuncts would not appear in the negative implicates due to the presence of their positive counterparts as atoms in $\PIp$. 
	The translation of the abduction problem $\langle\Tmc,\sig,C_1\sqsubseteq C_2\rangle$ is 
	defined as the
	Skolemization of
	$$\pi(\Tmc\uplus \negc{\Tmc})\land\neg\pi(C_1\isa \negc{C_2})$$
	where $\skz$ is used as the unique fresh Skolem constant such that the Skolemization of $\neg\pi(C_1\isa \negc{C_2})$
	results in $\{C_1(\skz),\neg \negc{C_2}(\skz)\}$.
  This translation is usually denoted $\Phi$ and always considered in clausal normal form.

  \begin{restatable}{theorem}{Construction}%
		\label{corr:constr}
		Let \mbox{$\langle\Tmc,\sig,C_1\isa C_2\rangle$} be an abduction problem and $\Phi$ be its first-order translation.
    Then, a TBox $\Hmc'$ is a packed connection-minimal solution to the problem if and only if an equivalent hypothesis $\Hmc$ can be constructed from non-empty sets $\Amc$ and $\Bmc$ of atoms verifying: %
		\begin{itemize}
    \item $\Bmc = \{B_1(t_1), \ldots, B_m(t_m)\}$ s.t.\ $\left( \neg \negc{B_1}(t_1)\vee\dots\vee \neg\negc{B_m}(t_m) \right)\in\PIn$,
    \item for all $t\in\{t_1,\ldots,t_m\}$ there exists an $A$ s.t.\ $A(t)\in\PIp$,
    \item $\Amc=\{A(t)\in\PIp\mid t\text{ is one of }t_1,\ldots,t_m\}$, and
    \item $\Hmc=\{C_{\Amc,t}\isa C_{\Bmc,t} \mid t\text{ is one of }t_1,\ldots,t_m \text{ and } C_{\Bmc,t}\not\smand C_{\Amc,t}\}$, %
      where $C_{\Amc,t}=\bigsqcap_{A(t)\in\Amc}A$ and $C_{\Bmc,t}=\bigsqcap_{B(t)\in\Bmc}B$.
		\end{itemize}
  \end{restatable}
  We call the hypotheses that are constructed as in Th.\ \ref{corr:constr} \emph{constructible}.
  This theorem states that every packed connection-minimal hypothesis is equivalent to
  a constructible hypothesis and vice versa. A constructible hypothesis is built from
  the concepts in \emph{one} negative prime implicate in $\PIn$ and \emph{all}
  matching concepts from prime implicates in $\PIp$.
  The matching itself is determined by the Skolem terms that occur in all these clauses.
  The subterm relation between the terms of the clauses in $\PIp$ and $\PIn$ is the same as the ancestor relation in the description trees of subsumers of $C_1$ and subsumees of $C_2$ respectively. 
  The terms matching in positive and negative prime implicates allow us to identify where the missing entailments between a subsumer $D_1$ of $C_1$ and a subsumee $D_2$ of $C_2$ are.
These missing entailments become the constructible $\Hmc$. %
   The condition $C_{\Bmc,t}\not\smand C_{\Amc,t}$ is a way to write that $C_{\Amc,t}\isa C_{\Bmc,t}$ is not a tautology, which can be tested by subset inclusion.

The formal proof of this result is detailed in App\ \ref{app:cons}.\footnote{In case of acceptance, the long version of this paper, including appendices, will be made available on an open public archive, such as arXiv.}
We sketch it briefly here.
To start, we link the subsumers of $C_1$ with $\PIp$.
This is done at the semantics level:
We show that all Herbrand models of $\Phi$, i.e., models built on the symbols in $\Phi$, are also models of $\PIp$, that is itself such a model.
Then we show that  $C_1(\skz)$ as well as the formulas corresponding to the subsumers of $C_1$ in our translation are satisfied by all Herbrand models.
This follows from the fact that $\Phi$ is in fact a set of Horn clauses.
Next, we show, using a similar technique, how duplicate negative ground implicates, not necessarily prime, relate to subsumees of $C_2$,
with the restriction that there must exist a weak homomorphism from a description tree of a subsumer of $C_1$ to
  a description tree of the considered subsumee of $C_2$. %
  Thus, $\Hmc$ provides the missing CIs that will turn the weak homomorphism into a $(\Tmc\cup\Hmc)$-homomorphism.
  Then, we establish an equivalence between the \smand-minimality of the subsumee of $C_2$ and the primality of the corresponding negative implicate.
  Packability is the last aspect we deal with, whose use is purely limited to the reconstruction.
  It holds because \Amc contains all $A(t)\in\PIp$ for all terms $t$ occurring in \Bmc.

  \begin{example}
    \label{ex:constr}
  Consider the abduction problem $\langle \Tac,\Sigma, \Obsa \rangle$ where $\Sigma$ contains all concepts from $\Tac$.
  For the translation $\Phi$ of this problem, we have
  \begin{eqnarray*}
    \PIp &=& \{\, \Prf(\skz),\, \Doc(\skz),\, \Chr(\sk_1(\skz)),\, \PhD(\sk_2(\skz))\}\\
    \PIn &=& \{\,\neg\negc{\Res}(\skz),\\
    &&\;\;\, \neg\negc{\RPo}(\sk_1(\skz))\vee\neg\negc{\Dip}(\sk_2(\skz))\}
  \end{eqnarray*}
  where $\sk_1$ is the Skolem function introduced for %
  $\Prf\isa\exists\emp.\Chr$ and $\sk_2$ is introduced for
  $\Doc\isa\exists\qual.\PhD$.
  This leads to two constructible solutions: $\{\Prf\dlAnd\Doc\isa\Res\}$ and $\Hac$,
  that are both packed connection-minimal hypotheses if $\sig=\NC$.
  Another example is presented in full details in App\ \ref{app:trans}.
  \end{example}

\paragraph{\textbf{Termination.}}
If $\Tmc$ contains cycles, there can be infinitely many prime implicates.
For example, for $\Tmc=\{C_1\isa A, A\isa\exists r.A, \exists r. B\isa B, B\isa C_2\}$ %
both the positive and negative ground prime implicates of $\Phi$ are unbounded even though the set of constructible hypotheses is finite (as it is for any abduction problem):
  \begin{eqnarray*}
    \PIp &=& \{C_1(\skz),A(\skz), A(\sk(\skz)), A(\sk(\sk(\skz))), \ldots\},\\
    \PIn &=& \{\neg\negc{C_2}(\skz),\neg\negc{B}(\skz),\neg\negc{B}(\sk(\skz)),\ldots\}.
  \end{eqnarray*}
  To find all constructible hypotheses of an abduction problem, an approach that simply computes all prime implicates of $\Phi$, e.g., using the standard resolution calculus, will never terminate on cyclic problems.
  However, if we look only for subset-minimal constructible hypotheses, termination can be achieved for cyclic and non-cyclic problems alike, because it is possible to construct all such hypotheses from prime implicates that have a polynomially bounded term depth, as shown below.
  To obtain this bound, we consider resolution derivations of the ground prime implicates and we show that they can be done under some restrictions that imply this bound.
  Before performing resolution, we compute the \emph{presaturation $\Phi_p$ of the set of clauses $\Phi$}, defined
  as
  \[
    \Phi_p=\Phi\cup\{\neg A(x)\vee B(x)\mid \Phi\models\neg A(x)\vee B(x)\}
  \]
  where $A$ and $B$ are either both original or both duplicate atomic concepts.
  The presaturation can be efficiently computed before the translation, using a modern \EL reasoner such as
  \ELK~\cite{ELK}, which is highly optimized towards the computation of all entailments of the form
  $A\sqsubseteq B$. While the presaturation computes nothing a resolution procedure could not derive, it is what allows us to bind the maximal depth of terms in inferences to that in prime implicates.
  If $\Phi_p$ is presaturated, we do not need to perform inferences that
  produce Skolem terms of a higher nesting depth than what is needed for the prime implicates.

  Starting from the presaturated set $\Phi_p$, we can show that all the relevant prime implicates can be
  computed
  if we restrict all inferences to those where
  \begin{enumerate}[label=\textbf{R\arabic*}]
   \item \label{it:oneground} at least one premise contains a ground term,
   \item \label{it:onevar} the resolvent contains at most one variable, and
   \item \label{it:bound} every literal in the resolvent contains Skolem terms of nesting depth at most $n\times m$, where $n$ is the number of atomic concepts in $\Phi$, and $m$ is the number of occurrences of existential role restrictions in $\Tmc$.
  \end{enumerate}
  The first restriction turns the derivation of $\PIp$ and $\PIn$ into an SOS-resolution derivation \cite{DBLP:conf/cade/HaifaniTW21} with set of support $\{C_1(\skz),\negc{C_2}(\skz)\}$, i.e., the only two clauses with ground terms in $\Phi$.
  This restriction is a straightforward consequence of our interest in computing only ground implicates, and of the fact that the non-ground clauses in $\Phi$ cannot entail the empty clause since every \EL TBox is consistent.
  The other restrictions are consequences of the following theorems, whose proofs are available in App.\ \ref{app:term}.

  \begin{restatable}{theorem}{ThmVariableBound}
    \label{thm:variable-bound}
  Given an abduction problem and its translation $\Phi$, every constructible hypothesis can be built 
  from prime implicates that are inferred under restriction~\ref{it:onevar}.
\end{restatable}
In fact, for $\PIp$ it is even possible to restrict inferences to generating only ground resolvents, as can be seen in the proof of Th.\ \ref{thm:variable-bound}, that directly looks at the kinds of clauses that are derivable by resolution from $\Phi$.

\begin{restatable}{theorem}{ThmBound}\label{the:bound}
  Given an abduction problem and its translation $\Phi$, every subset-minimal constructible hypothesis can be built from prime implicates that have a nesting depth of at most $n\times m$, where $n$ is the number of atomic concepts in $\Phi$, and $m$ is the number of occurrences of existential role restrictions in $\Tmc$.
  \end{restatable}
  The proof of Th.~\ref{the:bound} is based on a structure called a \emph{solution tree}, which resembles a description tree, but with multiple labeling functions.
  It assigns to each node a Skolem term, a set of atomic concepts called \emph{positive label}, and a single atomic concept called \emph{negative label}.
  The nodes correspond to matching partners in a constructible hypothesis: The Skolem term is the term
  on which we match literals. The positive label collects the atomic concepts in the positive prime implicates
  containing that term.
  The maximal anti-chains of the tree, i.e., the maximal subsets of nodes s.t.\ no node is the ancestor of another are such that their negative labels correspond to the literals in a derivable negative implicate. %
  For every solution tree, the Skolem labels and negative labels of the leaves determine a negative prime implicate, and by combining the positive and negative labels of these leaves, we obtain a constructible hypothesis, called the \emph{solution} of the tree. %
  We show that from every solution tree with solution $\Hmc$ we can obtain a solution tree 
  with solution $\Hmc'\subseteq\Hmc$ s.t.\ on no path, there are two nodes that agree both on the
  head of their Skolem labeling and on the negative label.
  Furthermore the number of head functions of Skolem labels is bounded by the total number $n$ of Skolem functions, while the number of distinct negative labels is bounded by the number $m$ of atomic concepts, bounding the depth of the solution tree for $\Hmc'$ at $n\times m$.
  This justifies the bound in Th~\ref{the:bound}.
  This bound is rather loose.
  For the academia example, it is equal to $22\times 6 = 132$.

\section{Implementation}
\label{sec:implem}

We implemented our method to compute all subset-minimal constructible
hypotheses in the tool CAPI.\footnote{available under
\url{https://lat.inf.tu-dresden.de/~koopmann/CAPI}}
To compute the prime implicates, we used SPASS \cite{DBLP:conf/cade/WeidenbachSHRT07}, a first-order theorem prover that includes resolution among other calculi.
We implemented everything before and after the prime implicate computation in Java, including the parsing of ontologies, preprocessing (detailed below), clausification of the abduction problems, translation to SPASS input, as well as the parsing and processing of the output of SPASS to build the constructible hypotheses and filter out the non-subset-minimal ones.
On the Java side, we used the OWL API for all DL-related functionalities~\cite{OWL-API}, and the \EL reasoner \ELK for computing the presaturations~\cite{ELK}.

 \paragraph{Preprocessing.}
 Since realistic TBoxes can be too large to be processed by SPASS,
 we replace
 the background knowledge in the abduction problem by a subset of axioms relevant to the abduction problem. Specifically,
 we replace the abduction problem $\tup{\Tmc,\Sigma,C_1\sqsubseteq C_2}$ by the abduction
 problem $\tup{\Mmc_{C_1}^\bot\cup\Mmc_{C_2}^\top,\Sigma,C_1\sqsubseteq C_2}$, where
 $\Mmc_{C_1}^\bot$ is the \emph{$\bot$-module} of $\Tmc$ for the signature of $C_1$, and $\Mmc_{C_2}^\top$
 is the \emph{$\top$-module} of $\Tmc$ for the signature of $C_2$ \cite{DBLP:journals/jair/GrauHKS08}.
 Those notions are explained in App\ \ref{app:modules}.
 Their relevant properties are that $\Mmc_{C_1}^\bot$ is a subset of $\Tmc$ s.t.\ %
 $\Mmc_{C_1}^\bot\models C_1\sqsubseteq D$ iff $\Tmc\models C_1\sqsubseteq D$ for all concepts
 $D$, while $\Mmc_{C_2}^\top$ is a subset of $\Tmc$ that ensures
 $\Mmc_{C_2}^\top\models D\sqsubseteq C_2$ iff $\Tmc\models D\sqsubseteq C_2$ for all concepts
 $D$. It immediately follows %
 that %
 every connection-minimal hypothesis for the original problem
 $\tup{\Tmc,\Sigma,C_1\sqsubseteq C_2}$ is also a connection-minimal hypothesis for
 $\tup{\Mmc_{C_1}^\bot\cup\Mmc_{C_2}^\top,\Sigma,C_1\sqsubseteq C_2}$.
 For the presaturation, %
 we compute with \ELK all CIs of the form $A\isa B$ s.t.\ $\Mmc_{C_1}^\bot\cup\Mmc_{C_2}^\top\models A\isa B$. %

 \paragraph{Prime implicates generation.}
 We rely on a slightly modified version of SPASS v3.9 to compute all ground prime implicates.
 In particular, we added the possibility to limit the number of variables allowed in the resolvents to enforce\ \ref{it:onevar}. %
 For each of the restrictions\ \ref{it:oneground} -- \ref{it:bound} there is a corresponding flag (or set of flags) that is passed to SPASS as an argument.

\paragraph{Recombination.}
The construction of hypotheses from the prime implicates found in the previous stage starts with a straightforward process of matching negative prime implicates with a set of positive ones based on their Skolem terms.
It is followed by subset minimality tests to discard non-subset-minimal hypotheses, since, with the bound we enforce, there is no guarantee that these are valid constructible hypotheses because the negative ground implicates they are built upon may not be prime.
If SPASS terminates due to a timeout instead of reaching the bound, then it is possible that some subset-minimal constructible hypotheses are not found, and thus, some non-constructible hypotheses may be kept.
Note that these are in any case solutions to the abduction problem.

\section{Experiments}
\label{sec:exp}

\newcommand{\NonEntailed}{\texttt{ORIGIN}\xspace}%
\newcommand{\Just}{\texttt{JUSTIF}\xspace}%
\newcommand{\Entailed}{\texttt{REPAIR}\xspace}%

There is no benchmark suite dedicated to TBox abduction in \EL, so we created
our own, using realistic ontologies from the bio-medical domain.
For this, we used ontologies from the
2017 snapshot of Bioportal~\cite{Bioportal-2017}.
We restricted each ontology to its \EL fragment by filtering out
unsupported axioms, where we replaced domain axioms and n-ary equivalence axioms in the usual way \cite{DL_TEXTBOOK}.
Note that, even if the ontology contains more expressive axioms,
an \EL hypothesis is still useful if found.
From the resulting set of TBoxes, we selected those containing at least 1 and at most 50,000 axioms, resulting in a set of 387 \EL TBoxes.
Precisely, they contained between 2 and 46,429 axioms, for an average of 3,039 and a median of 569.
Towards obtaining realistic benchmarks, we created three different categories of abduction problems for each ontology $\Tmc$, where in each case, we used the
signature of the entire ontology for $\sig$.
\begin{itemize}
 \item Problems in \NonEntailed use $\Tmc$ as background knowledge, and as observation a randomly chosen $A\sqsubseteq B$ s.t.\ $A$ and $B$ are in the signature of $\Tmc$ and $\Tmc\not\models A\sqsubseteq B$. This covers the basic requirements of an abduction problem, but has the disadvantage that $A$ and $B$ can be completely unrelated in $\Tmc$.

 \item Problems in \Just contain as observation a randomly selected CI $\alpha$ s.t., for the original TBox, $\Tmc\models\alpha$ and $\alpha\not\in\Tmc$. The background knowledge used is a \emph{justification for $\alpha$ in $\Tmc$}~\cite{ScCo03}, that is, a minimal subset $\Jmc\subseteq\Tmc$ s.t.\ $\Jmc\not\models\alpha$, from which a randomly selected axiom is removed. The TBox is thus a smaller set of axioms extracted from a real ontology for which we know there is a way of producing the required entailment without adding it explicitly. Justifications were computed using functionalities of the OWL API and \ELK.

 \item Problems in \Entailed contain as observation a randomly selected CI $\alpha$ s.t.\ $\Tmc\models\alpha$, and as background knowledge a
 \emph{repair for $\alpha$ in $\Tmc$}, which is a maximal subset $\Rmc\subseteq\Tmc$ s.t. $\Rmc\not\models\alpha$. Repairs were computed using a justification-based algorithm \cite{ScCo03} with justifications computed as for~\Just. This usually resulted in much larger TBoxes, where more axioms would be needed to establish the entailment.
\end{itemize}

All experiments were run on Debian Linux (Intel Core i5-4590, 3.30\,GHz, 23\,GB Java heap size).
The code and scripts used in the experiments are available
online~\cite{ZENODO_LINK}.
The three phases of the method (see Fig.\ \ref{fig:birdeye}) were each assigned a hard time limit of 90 seconds.

For each ontology, we attempted to create and translate 5 abduction problems of each category. %
This failed on some
ontologies because either there was no corresponding entailment
(25/28/25 failures out of the 387 ontologies for \NonEntailed/\Just/\Entailed),
there was a timeout during the translation (5/5/5 failures for \NonEntailed/\Just/\Entailed),
or because the computation of justifications caused an exception
(-/2/0 failures for \NonEntailed/\Just/\Entailed).
The final number of abduction problems for each category is in the first column of Table~\ref{tab:results}.

We then attempted to compute prime implicates for these benchmarks using SPASS.
In addition to the hard time limit, we gave a soft time limit of 30 seconds to SPASS,
after which it should stop exploring the search space and return the implicates already found.
In Table\ \ref{tab:results} we show, for each category, the percentage of problems on which SPASS succeeded in computing a non-empty set of clauses (Success) and the percentage of problems on which SPASS terminated within the time limit, where all solutions are computed (Compl.).
The high number of CIs in the background knowledge explains most of the cases where SPASS reached the soft time limit.
In a lot of these cases, the bound on the term depth goes into the billion, rendering it useless in practice.
However, the ``Compl.'' column shows that the bound is reached before the soft time limit in most cases.

The reconstruction never reached the hard time limit.
We measured the median, average and maximal number of solutions found (\#\Hmc), size of solutions in number of CIs ($|\Hmc|$), size of CIs from solutions in number of atomic concepts ($|\alpha|$), and SPASS runtime (time, in seconds), all reported in Table\ \ref{tab:results}.
Except for the simple \Just problems, the number of solutions may become very large.
At the same time, solutions always contain very few axioms (never more than 3), though the axioms become large too.
We also noticed that
highly nested Skolem terms rarely lead to more hypotheses being found: 8/1/15 for \NonEntailed/\Just/\Entailed,
and the largest nesting depth used was: 3/1/2 for \NonEntailed/\Just/\Entailed. %
This hints at the fact that longer time limits would not have produced more solutions, and motivates future research into redundancy criteria to stop derivations (much) earlier.

\begin{table}[t]
  \begin{tabular}{r || r || r | r || c | c | c | c}
    & & & & \multicolumn{4}{c}{ median / avg / max} \\
     &\#Probl. & Success & Compl. & \#\Hmc& $|\Hmc|$ & $|\alpha|$ & time (s.)\\ %
    \hline
    \NonEntailed & 1,925 & 94.7\% &  61.3\% &
    1/8.51/1850 &
    1/1.00/2 &
    6/7.48/91 &
    0.2/12.4/43.8%
    \\
    \Just        & 1,803 & 100.0\% & 97.2\% &
    1/1.50/5 &
    1/1/1 &
    2/4.21/32 &
    0.2/1.1/34.1\\
    \Entailed    & 1,805 & 92.9\% & 57.0\% &
    43/228.05/6317 &
    1/1.00/2 &
    5/5.09/49 &
    0.6/13.6/59.9
    \\  %
  \end{tabular}

  \bigskip

  \caption{Evaluation results.}
  \label{tab:results}
\end{table}

\section{Conclusion}

We have introduced connection-minimal TBox abduction for \EL which
finds parsimonious hypotheses,
ruling out the ones that entail the observation in an arbitrary fashion. We have
established a formal link between the generation of connection-minimal
hypotheses in \EL and the generation of prime implicates of a translation $\Phi$
of the problem to first-order logic. In addition to obtaining these theoretical results, we
developed a prototype for the computation of subset-minimal constructible
hypotheses, a subclass of connection-minimal hypotheses that is easy to
construct from the prime implicates of $\Phi$. Our prototype uses the SPASS
theorem prover as an SOS-resolution engine to generate the needed
implicates. We tested this tool on a set of realistic medical ontologies, and
the results indicate that the cost of computing connection-minimal hypotheses is
high but not prohibitive.

We see several ways to improve our technique. The bound we computed to ensure
termination could be advantageously replaced by a redundancy criterion
discarding irrelevant implicates long before it is reached, thus greatly
speeding computation in SPASS.
We believe it should also be possible to further constrain
inferences, e.g., to have them produce ground clauses only, or to generate the prime
implicates with terms of increasing depth in a controlled incremental way
instead of enforcing the soft time limit, but these two ideas remain to be
proved feasible. As an alternative to using prime implicates, one may
investigate direct method for computing connection-minimal hypotheses in \EL.

The theoretical worst-case complexity of connection-minimal abduction is another
open question. Our method only gives a very high upper bound: by
bounding only the nesting dept of Skolem terms polynomially as we did with
Th.~\ref{thm:variable-bound}, we may still permit clauses with exponentially
many literals, and thus double exponentially many clauses in the worst case,
which would give us an \TwoExpTime upper bound to the problem of computing all
subset-minimal constructible hypotheses. Using structure-sharing and guessing,
it is likely possible to get a lower bound. We have not looked yet at lower
bounds for the complexity either.

While this work focuses on abduction problems where the observation is a CI, we believe that our technique can be generalised to knowledge that
also contains ground facts (ABoxes), and to observations that are of the form of
conjunctive queries on the ABoxes in such knowledge bases. The motivation for such an extension
is to understand why a particular query does not return any results, and
to compute a set of TBox axioms that fix this problem. Since our translation
already transforms the observation into ground facts, it should be possible to
extend it to this setting. We would also like to generalize TBox abduction by
finding a reasonable way to allow role restrictions in the hypotheses, and to
extend connection-minimality to more expressive DLs such as \ALC.

\section*{Acknowledgments}
This work was supported by 
the Deutsche Forschungsgemeinschaft (DFG), 
Grant 389792660 within TRR~248.

\bibliographystyle{splncs04}
\bibliography{bib}

\newpage

\appendix

\section{Various Minor Results}

\subsection{\Tmc-homomorphism and Entailment}
\label{app:homo-ent}
\begin{lemma}
  Let $\Tr_1=(V_1,E_1, v_0, l_1)$ and $\Tr_2=(V_2,E_2,w_0,l_2)$ be \EL description trees, with a $\Tmc$-homomorphism $\Phi$ from $\Tr_2$ to $\Tr_1$.
  Then $\Tmc\models C_{\Tr_1}\isa C_{\Tr_2}$.
\end{lemma}
\begin{proof}
We prove this result by induction on the structure of $\Tr_2$.

If $\Tr_2 = (\{w_0\}, \emptyset,$ $ w_0, l_2)$, then $C_{\Tr_2}= l_2(w_0)$.
Moreover $C_{\Tr_1} \isa l_1(v_0)$ by  definition of $C_{\Tr_1}$.
Finally $\Tmc\models C_{\Tr_1}\isa C_{\Tr_2}$ since $\Tmc \models
\bigsqcap l_1(\phi(w_0))\isa \bigsqcap l_2(w_0)$ and $\phi(w_0)=v_0$.

In the general case, let us consider any child $w_i$ of $w_0$ in $\Tr_2$ since there must be at least one.
Then there is a corresponding child $v_i$ of $v_0$ in $\Tr_1$ s.t.\ $v = \phi(w)$.
The $\Tmc$-homomorphism $\phi$ from $\Tr_2$ to $\Tr_1$ is also a $\Tmc$-homomorphism from $\Tr_2(w)$ to $\Tr_1(v)$, thus by induction $\Tmc\models C_{\Tr_1(v)}\isa C_{\Tr_2(w)}$.
And in particular, for the $r_i$ such that $w_0r_iw\in E_2$, we have $\Tmc\models \exists r_i.C_{\Tr_1(v)}\isa \exists r_i.C_{\Tr_2(w)}$.
This applies to all the children $w_1,\dots,w_n$ of $w_0$, and since $\Tmc \models \bigsqcap l_1(v_0)\isa \bigsqcap l_2(w_0)$, it follows that $\Tmc\models \bigsqcap C_{\Tr_1}\isa C_{\Tr_2}$.
\end{proof}

\subsection{Straightforward Consequences of Definition \ref{def:abd}}
\label{app:straight_cons}
Point \ref{it:H} of Def.~\ref{def:abd} turns $\phi$ from a weak homomorphism to a \Tmc-homomorphism.
The hypothesis \Hmc\ is made to add exactly the entailments that are missing in \Tmc\ to ensure that point \ref{it:entails} of Def.~\ref{def:homomorph} is satisfied.
Thanks to this $\Tmc\cup\Hmc\models D_1\isa D_2$ and thus both $D_1$ and $D_2$ become connecting concepts from $C_1$ to $C_2$ in $\Tmc\cup\Hmc$, because we have $\Tmc\cup\Hmc\models C_1\isa D_1\isa D_2\isa C_2$.

\subsection{Normalization}
\label{app:preliminaries}
\LemNormalize*
\begin{proof}
    Most of the lemma is well-known. How normalization can be performed, and
that it is possible in polynomial time, is shown in
~\cite[Lemma 6.2]{DL_TEXTBOOK}. Furthermore,
by~\cite[Proposition~6.5]{DL_TEXTBOOK}, the result of this transformation is a
conservative extension $\Tmc'$ of the original TBox $\Tmc$ in the sense that:
\begin{enumerate}
 \item $\Tmc'\models\Tmc$, and
 \item for every model $\Imc$ of $\Tmc$, there exists a model $\Imc'$ of $\Tmc'$
s.t. for every concept name $A$ occurring in $\Tmc$, $A^\Imc=A^{\Imc'}$, and
for every role name $r\in\NR$ occurring in $\Tmc$, $r^\Imc=r^{\Imc'}$.
\end{enumerate}

Now let $\Hmc$ be a TBox and $C\sqsubseteq D$ a CI such that both only use
names occurring in $\Tmc$. If $\Tmc\cup\Hmc\models C\sqsubseteq D$, we observe
that by Item~1, we have $\Tmc'\cup\Hmc\models\Tmc\cup\Hmc$, and thus by
transitivity of entailment, $\Tmc'\cup\Hmc\models C\sqsubseteq D$. Assume
$\Tmc\cup\Hmc\not\models C\sqsubseteq D$. Then there exists a model $\Imc$ of
$\Tmc\cup\Hmc$ s.t. $\Imc\not\models C\sqsubseteq D$. Since $\Hmc$, $C$ and $D$
only use names occurring in $\Tmc$, by Item~2, we can find a model $\Imc'$ of
$\Tmc$ s.t. $\Imc'\models\Hmc$ and $\Imc'\not\models C\sqsubseteq D$, and
consequently, $\Tmc'\cup\Hmc\not\models C\sqsubseteq D$. We obtain that
$\Tmc\cup\Hmc\models C\sqsubseteq D$ iff $\Tmc'\cup\Hmc\models C\sqsubseteq D$.
\qed
\end{proof}

\section{Detailed Example}
\label{app:trans}

Consider the abduction problem $\langle \Tmc,\Sigma, C_1\isa C_2 \rangle$ where
\begin{eqnarray*}
  \Tmc &=& \{\,C_1\isa H,\, C_1\isa L,\, C_1\isa\exists r_1. A,\, A\isa\exists r_2. M,\, A\isa\exists r_2.B,\\
       && \;\; \exists r_1.X\sqcap E\isa C_2,\, F\sqcap Y\isa X,\, \exists r_2. M\sqcap\exists r_2. Z\isa Y,\, G\sqcap H\isa Z\,\}
\end{eqnarray*}
and $\Sigma = \{A, B, D, E, F, G, H, L\}$.
Consider the concepts
\begin{align*}
  \hspace*{2cm}&D_1=&L\dlAnd H&\dlAnd\exists {r_1}.({A}\dlAnd\exists r_2.M\dlAnd\exists r_2.B),&\hspace*{1.5cm}\\
               &D_2=&E&\dlAnd\exists r_1.(F\dlAnd\exists r_2.M\dlAnd\exists r_2.(G\dlAnd K)),&
\end{align*}

Indeed, the concepts $D_1$ and $D_2$ are such that $\Tmc\models C_1\isa D_1$ and $\Tmc\models D_2\isa C_2$.
Moreover any $D_2'\smand D_2$ is not a subsumer of $C_2$ so $D_2$ is a \smand-minimal concept such that $\Tmc\models D_2\isa C_2$.
There is also a weak homomorphism from $\Tr_{D_2}$ to $\Tr_{D_1}$, as illustrated in Fig.\ \ref{fig:oldhomomorph}.
Thus,
$$\Hmc = \{L\sqcap H\isa E, A\isa F, B\isa G\sqcap K\}$$ is a connection-minimal hypothesis.
Note that the tautology $M\isa M$, that is one of the entailments that must hold in $\Tmc\cup\Hmc$, as is visible in Fig.\ \ref{fig:oldhomomorph}, is not included in \Hmc since it is a tautology and thus $\Tmc\models M\isa M$.  
The hypothesis \Hmc is even packed.
In contrast,
\begin{eqnarray*}
  \Hmc_1 &=& \{H\isa E,\, A\isa F,\, B\isa G\sqcap H\},\text{ and}\\
  \Hmc_2 &=& \{L\isa E,\, A\isa F,\, B\isa G\sqcap H\},
\end{eqnarray*}
that are both connection-minimal but lack either $L$ or $H$ on the left-hand side of their first CI when compared with \Hmc, are not packed.

  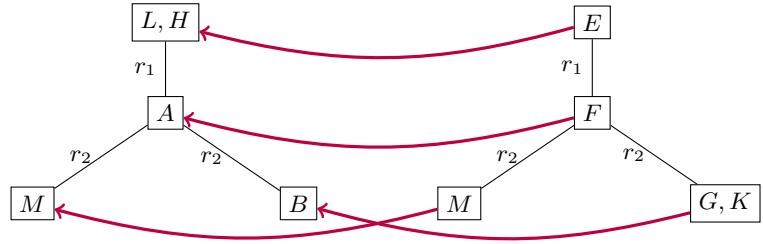
\begin{figure}[b]
    \centering
    \begin{tikzpicture}[level distance=1.2cm,
			level 1/.style={sibling distance=11cm},
			level 2/.style={sibling distance=3.5cm}]
			
			\node{}
			child{edge from parent[draw=none] node(v0)[draw] {$L,H$}
				child { node(v1)[draw] {{$A$} }  
					child { node(v2)[draw] {{$M$}}edge from parent node[left] {{$r_2$}}
					}
					child { node(v3)[draw] {{$B$}}edge from parent node[left] {{$r_2$}}
					}
					edge from parent node[left] {{$r_1$}}
				}
			}
			child{edge from parent[draw=none] node(w0)[draw] {$E$}
				child { node(w1)[draw] {$F$ }  
					child { node(w2)[draw] {$M$}edge from parent node[left] {$r_2$}
					}
					child { node(w3)[draw] {$G,K$}edge from parent node[left] {$r_2$}
					}
					edge from parent node[left] {$r_1$}
				}
			};
			\draw[color=purple,->,very thick] (w0) to [bend left=15] (v0);
			\draw[color=purple,->,very thick] (w1) to [bend left=15] (v1);
			\draw[color=purple,->,very thick] (w2) to [bend left=15] (v2);
			\draw[color=purple,->,very thick] (w3) to [bend left=15] (v3);
    \end{tikzpicture}
    \caption{Two description trees with a weak homomorphism between them.}
    \label{fig:oldhomomorph}
  \end{figure}

We apply our technique to compute the hypotheses of this abduction problem.
Since $\Tmc$ is not in normal form, we must normalize it before the translation.
The CIs to normalize are $\exists r_1.X\sqcap E\isa C_2$ and $\exists r_2. M\sqcap\exists r_2. Z\isa Y$ for which we introduce the fresh concepts $U$, $V$ and $W$ and corresponding CIs $\exists r_2.M\isa U$, $\exists r_2.Z\isa V$ and $\exists r_1.X\isa W$, along with the normalized form of the two initial CIs, i.e., $W\sqcap E\isa C_2$ and $U\sqcap V\isa Y$.

Since $\Tmc$ does not contain $\top$, there is no need to introduce $A_\top$.
We write the CIs in the normalization of $\Tmc$ and their first-order translation after Skolemization side by side.
$$
\begin{array}{r@{\qquad}l}
C_1 \isa  H & \neg C_1(x)\vee H(x)\\[.4em]
C_1 \isa L & \neg C_1(x)\vee L(x)\\[.4em]
C_1\isa\exists r_1. A &\left\{\begin{array}{l}
 \neg C_1(x)\vee r_1(x,\sk_1(x))\\
 \neg C_1(x)\vee A(\sk_1(x))\\
 \end{array}\right.\\[1em]
A\isa\exists r_2. M&\left\{\begin{array}{l}
 \neg A(x)\vee r_2(x,\sk_2(x))\\
 \neg A(x)\vee M(\sk_2(x))\\
 \end{array}\right.\\[1em]
A\isa\exists r_2.B&\left\{\begin{array}{l}
 \neg A(x)\vee r_2(x,\sk_2'(x))\\
 \neg A(x)\vee B(\sk_2'(x))\\
 \end{array}\right.\\[1em]
\exists r_1.X\isa W & \neg r_1(x,y)\vee\neg X(y)\vee W(x)\\[.4em]
W\sqcap E\isa C_2 & \neg W(x)\vee\neg E(x)\vee C_2(x)\\[.4em]
F\sqcap Y\isa X & \neg F(x)\vee\neg Y(x)\vee X(x)\\[.4em]
\exists r_2. M\isa U & \neg r_2(x,y)\vee\neg M(y)\vee U(x)\\[.4em]
\exists r_2. Z\isa V & \neg r_2(x,y)\vee\neg Z(y)\vee V(x)\\[.4em]
U\sqcap V\isa Y & \neg U(x)\vee\neg V(x)\vee Y(x)\\[.4em]
G\sqcap H\isa Z & \neg G(x)\vee\neg H(x)\vee Z(x)
\end{array}
$$
 The translation of $\negc{\Tmc}$ is identical to that of $\Tmc$ up to the replacement of every unary predicate with its duplicate and the introduction of fresh Skolem functions distinct from the ones used for \Tmc. 
 Let $\Phi$ denote the full translation of the problem.
 The ground prime implicates for $\Sigma=\{A,B,E,F,G,H,L,M\}$ are as follows:
  \begin{eqnarray*}
    \PIp &=& \{\, L(\skz),\, H(\skz),\, A(\sk_1(\skz)),\, M(\sk_2(\sk_1(\skz))),\, B(\sk_2'(\sk_1(\skz)))\}\\
    \PIn &=& \{\,\neg\negc{E}(\skz)\vee\neg\negc{F}(\sk_1(\skz))\vee\neg\negc{M}(\sk_2(\sk_1(\skz)))\vee\\
    && \qquad \qquad \qquad \qquad \qquad \neg\negc{G}(\sk_2'(\sk_1(\skz)))\vee\neg\negc{H}(\sk_2'(\sk_1(\skz)))\}
  \end{eqnarray*}
  where $\sk_1$ is the Skolem function corresponding to the existential quantifier introduced by the translation of $\exists r_1.X$ to first-order logic, and where $\sk_2$ and $\sk_2'$ correspond respectively to $\exists r_2. M$ and $\exists r_2. Z$.
  The only constructible hypothesis out of this configuration is \Hmc, the packed connection-minimal hypothesis already introduced.
    Finally, $\Hmc_3$ is the smallest TBox that fixes all entailments missing between $D_1$ and $D_2$, ensuring the connection minimality of $\Hmc_3$ and it is packed, contrarily to $\Hmc_1$ and $\Hmc_2$ that lack either $L$ or $H$ on the left-hand side of their first CI.
  Note that the signature restriction has been made to capture only these solutions, but there would be many more if we considered the whole signature after normalization for $\Sigma$.
  In particular, including $A_\top$ to $\Sigma$ would produce all solutions where $\top$ replaces the left-hand side of some CIs in another solution, so it should generally be avoided.

\section{Construction}
\label{app:cons}
We recall the statement of the main theorem that we prove in this appendix:

\Construction*

Until the end of this appendix, we assume $\Sigma=\NC$.
For Th.\ \ref{corr:constr}, the case where $\Sigma\subsetneq\NC$ trivially follows, but that is not the case for the intermediate results.

   To prove Th.\ \ref{corr:constr}, we first establish the link between the positive prime implicates in $\PIp$ and the subsumers of $C_1$, then we do the same for the negative side.
First, we adapt the notion of a canonical model by Baader et al.\ \cite{baader1999computing}, to construct a minimal Herbrand Interpretation $\CM$ ensuring that, for a given concept $C$, at least one Skolem term $t$ is such that $t\in C^\CM$.
We show how to extend a canonical model so that it also satisfies $\Tmc$ and we link the existence of such a model built for some $C_\Tr$ and $\skz$ to the existence of the entailment $\Tmc\models C_1\isa C_\Tr$, while showing that this model is in fact a subset of \PIp.

  Second, we show how renamed negative ground implicates, not necessarily prime, relate to the subsumees of $C_2$.
  To that aim, we again rely on a canonical model, but this time for the renamed version of some $C_\Tr$ subsumee of $C_2$, with the restriction that
there must exist a weak homomorphism from a subsumer of $C_1$ to
  this $C_\Tr$, %
  the idea being that $\Hmc$ is built to provide the missing CIs that will turn the weak homomorphism into a $(\Tmc\cup\Hmc)$-homomorphism.

  Finally, we establish an equivalence between the \smand-minimality of $C_\Tr$ and the fact that the corresponding renamed negative implicate is prime.

   Before diving into the proofs, remember that we work under the assumptions that $\Tmc$ is in normal form and without $\top$, and as a consequence $\pi(\Tmc)$ contains only axioms of the following shapes:
   \begin{align*}
                                            &\forall x.\neg A(x)\lor B(x),\\
                                            &\forall x.\neg A_1(x)\lor \neg A_2(x)\lor B(x),\\
                                            &\forall x.\neg(\exists y.r(x,y)\land A(y))\lor B(x),\\
                                            &\forall x.\neg A(x)\lor\exists y.(r(x,y) \land B(y)).
   \end{align*}
   After Skolemization, the clauses are all Horn.

   We provide a direct specification of \EL semantics, that we work with in the proofs.
	It uses first-order structures or \emph{interpretations}, which are tuples
	$\Imc=\tup{\Delta^\Imc,\cdot^\Imc}$ made of a \emph{domain}
	$\Delta^\Imc$ and an interpretation function $\cdot^\Imc$ that maps atomic concepts
	$A\in\NC$ to sets $A^\Imc\subseteq\Delta^\Imc$ and roles $r\in\NR$ to relations
	$r^\Imc\subseteq\Delta^\Imc\times\Delta^\Imc$. The interpretation function $\cdot^\Imc$ is extended to complex
	concepts as follows:
	\begin{gather*}
	\top^\Imc=\Delta^\Imc \qquad (C\sqcap D)^\Imc=C^\Imc\cap D^\Imc\\
	(\exists r.C)^\Imc=\{d\in\Delta^\Imc\mid\exists\tup{d,e}\in r^\Imc\text{ s.t. } e\in C^\Imc \}
	\end{gather*}
	The interpretation $\Imc$ \emph{satisfies} a CI $C\isa D$, in symbols $\Imc\models C\isa D$, if
	$C^\Imc\subseteq D^\Imc$.
  If $\Imc$ satisfies all axioms in a TBox $\Tmc$,
	we write $\Imc\models\Tmc$ and call $\Imc$ a \emph{model} of $\Tmc$. If a CI $C\isa D$ is satisfied in every model
	of $\Tmc$, we write $\Tmc\models C\isa D$ and say that $C\isa D$ is \emph{entailed} by~$\Tmc$. In this case, we
	say that \emph{$D$ subsumes $C$}, or that \emph{$C$ is subsumed by $D$} and call $C$ a \emph{subsumee} of $D$ and $D$ a \emph{subsumer} of $C$.
	One easily verifies that the above
	translation of \EL axioms and TBoxes is consistent with their semantics, that is, that $\Imc\models\pi(\alpha)$ iff
	$\Imc\models\alpha$ for any CI $\alpha$ (or TBox)~\cite{Borgida94,BaaderEtAl2003handbook}.

  Let \NS denote the set of all monadic Skolem functions that are used to Skolemize the translation of an abduction problem to first-order logic.
  We call an interpretation $\Imc$ with $\Delta^\Imc=\gterms{\skz}$, where $\gterms{\skz}$ is the set of terms built on the constant \skz\ and functions from \NS, a \emph{Herbrand intepretation}, which for convenience, we identify with the set of ground atoms that are satisfied in it.
  Specifically, for a Herbrand interpretation $\Imc$, we write $A(t)\in\Imc$ if  $t\in A^\Imc$, and $r(t,t')\in\Imc$ if $(t,t')\in r^\Imc$.

  \subsection{Subsumers of $C_1$ and Positive Prime Implicates}
  
	We derive a relation between the subsumers of $C_1$ in $\Tmc$ and the prime implicates of $\Phi = \Pi(\Tmc,C_1\isa C_2)$. %
  This relation is established at the semantics level, by constructing a Herbrand model of $\Tmc$ and showing it necessarily contains the prime implicates of $\Phi$.

  For this purpose, we adapt the definition of a canonical model from \cite{baader1999computing} by using $\gterms{\skz}$ for the domain of the \EL-description tree $\Tr$ corresponding to a subsumer of $C_1$.\footnote{In \cite{baader1999computing}, the canonical interpretation uses the set of vertices as its domain.}
	\begin{definition}[Canonical Model]
    Given a description tree $\Tr=(V,E, v_0, l)$, a \emph{Skolem labeling} $sl_\Tr:V\rightarrow\gterms{\skz}$ of $\Tr$ maps the vertices of $\Tr$ to ground Skolem terms. A \emph{canonical model} $\CM(sl_\Tr)$ of $\Tr$ is a Herbrand interpretation consisting of the following atoms:
		\begin{itemize}
			\item $r(sl_\Tr(v),sl_\Tr(w))$ for all $vrw\in E$
			\item $A(sl_\Tr(v))$ for all $A\in l(v)$ and $v\in V$
		\end{itemize}
	\end{definition}
  We denote by $\CMA(sl_\Tr)$ the subset of $\CM(sl_\Tr)$ made of all atoms built over unary predicates, and by $\CMr(sl_\Tr)$, the rest of $\CM(sl_\Tr)$, that contains all atoms built over binary predicates.

  It is always possible to find a canonical model of an \EL-description tree $\Tr$ as a subset in any Herbrand interpretation $\Imc$ for which $(C_\Tr)^\Imc$ is not empty.
  This is formally stated, and proven, in the following lemma.
  
	\begin{lemma}%
		\label{lemma:ItoIc}
		Given an \EL-description tree $\Tr=(V,E, v_0, l)$ and a Herbrand interpretation $\Imc$,
		if $t\in (C_{\Tr})^{\Imc}$ then there exists a Skolem labeling $sl_{\Tr}$ s.t. $sl_{\Tr}(v_0)=t$ and $\CM(sl_{\Tr})\subseteq\Imc$
	\end{lemma}
	\begin{proof}
    Given an \EL-description tree $\Tr=(V,E, v_0, l)$, a Herbrand interpretation $\Imc$ and a Skolem term $t$, such that $t\in (C_{\Tr})^{\Imc}$, let us construct the suitable Skolem labeling $sl_\Tr$.
		We proceed inductively on the depth of $\Tr$. %
    $$ sl_\Tr(v)=
    \left\{
      \begin{array}{ll}
        t & \text{if }v = v_0,\\
        sl_{\Tr(w)}(v) & \text{if }v_0rw\in E\text{ for some }r,\text{ and }v\in V_{\Tr(w)},
      \end{array}
    \right.
    $$
    where $\Tr(w)$ is the subtree of $\Tr$ rooted in $w$, $V_{\Tr(w)}$ is the subset of $V$ that occurs in $\Tr(w)$ and $sl_{\Tr(w)}$ is defined as $sl_\Tr$ but on $t'$ instead of $t$, for a $t'$ such that $r(t,t')\in\Imc$ and $t'\in (C_{\Tr(w)})^\Imc$.
    Such a $t'$ must exist because $\exists r.C_{\Tr(w)}$ is a conjunct in $C_\Tr$ and $t\in(C_\Tr)^\Imc$.
    Hence $sl_{\Tr(w)}$ is well-defined.
    This construction terminates because the depth of all $\Tr(w)$ is strictly smaller than that of $\Tr$.

    If $\Tr$ is of depth 0, then $sl_\Tr$ is simply defined on $v_0$ such that $sl_{\Tr}(v_0)=t$, and $C_\Tr$ is a conjunction of atomic concepts $A\in l(v_0)$.
    Thus, $t\in (C_{\Tr})^{\Imc}$ is equivalent to $A(t)\in \Imc$ for all $A\in l(v_0)$.
    Hence, any atom $A(sl_\Tr(v_0))=A(t)\in \CM(sl_\Tr)$ is also in $\Imc$ for all $A\in l(v_0)$.

    If $\Tr$ is of depth $i>0$, for any $v\in V\setminus\{v_0\}$, there exists a $w\in\V$ such that $v_0rw\in E$ and $v\in V_{\Tr(w)}$, i.e., $v$ must belong to a subtree rooted in one of the children $w$ of the root of $\Tr$.
    Then $sl_\Tr(v)=sl_{\Tr(w)}(v)$.
    By induction, $\CM(sl_{\Tr(w)})\subseteq\Imc$.
    Moreover $A(t)\in\Imc$ for all $A\in l(v_0)$ as in the base case; and $r(t,sl_{\Tr(w)}(w))\in \Imc$ and $sl_{\Tr(w)}(w)\in (C_{\Tr(w)})^\Imc$ for all $w$ children of $v_0$ by construction of $sl_{\Tr(w)}$.
    Thus $\Imc(sl_\Tr)\subseteq\Imc$.
    \qedhere
	\end{proof}

  Now, we show that $\PIp$ holds the role of universal Herbrand model for $\Phi=\Pi(\Tmc, C_1\isa C_2)$.
  The proof adapts a result by Bienvenu et al. \cite{DBLP:conf/rweb/BienvenuO15} to the case with only one constant but a possibly infinite domain.

	\begin{lemma}[$\PIp$ as a universal model]
    \label{lem:pipmodel}
		Given the translation $\Phi$ of an abduction problem, the set $\PIp$ considered as a Herbrand interpretation is a model of $\Phi$ and for any other Herbrand model $\Imc$ of $\Phi$, $\PIp\subseteq \Imc$.
	\end{lemma}
	\begin{proof}
    By the definition of a prime implicate, any model of $\Phi$ must be a model of any $\varphi\in\PIp$.
    Moreover, a positive prime implicate can only be an atom since $\Phi$ contains only Horn clauses.
    Thus all Herbrand models must contain $\PIp$.

    To show that $\PIp$ is itself a Herbrand model, we construct the Herbrand Interpretation $\Imc = \bigcup_i \Imc_i$ for $i\in\mathbb{N}$ where:

    \begin{itemize}
    \item $\Imc_0=\{C_1(\skz)\}$ and,
    \item given $\Imc_j$,
      \begin{eqnarray*}
        \Imc_{j+1} &=& \Imc_j \cup \{B(t)\mid t\in (D)^{\Imc_j},\, \neg\pi(D,x)\vee B(x)\in\Phi\}\\
                   & &  \phantom{\Imc_j}\cup \{B(\sk(t)),r(t,\sk(t))\mid t\in (A)^{\Imc_j},\\
                   & & \qquad\quad\neg\pi(A,x)\vee B(\sk(x))\in \Phi,\, \neg\pi(A,x)\vee r(x,\sk(x))\in \Phi\}\\
                   & &  \phantom{\Imc_j}\cup \{B(t)\mid \sk(t)\in (A)^{\Imc_j},\, (t,\sk(t))\in r^{\Imc_j},\\
                   & & \qquad\quad\neg r(x,y)\vee\neg A(y)\vee B(x)\in\Phi\}.\\
      \end{eqnarray*}
    \end{itemize}

    We show that $\Imc$ is a model of $\Phi$.
    We know that $\Imc\models C_1(\skz)$ by construction of $\Imc_0$, and that all other clauses containing only non-duplicated literals are also satisfied by $\Imc$, again by construction.
    Note that there are cases where no $\Imc_j$ alone is enough to satisfies a clause, but they are all satisfied at the limit by $\Imc$ (e.g., if $\Tmc$ includes a concept inclusion $A\isa\exists r. A$, possibly leading to the presence of infinitely many atoms of the form $A(sk^n(t))\in\Imc$).
    Regarding the remaining clauses in $\Phi$, they all contain at least one literal of the form $\neg \negc{A}(x)$ and since $\Imc$ includes no atom $\negc{A}(t)$ at all, $\Imc$ also satisfies that part, thus $\Imc$ is a model of $\Phi.$

    It remains only to show that $\Imc\subseteq \PIp$ to have $\PIp = \Imc$, thus showing that $\PIp$ is a model of $\Phi$.
    This is done by induction.
    Clearly $\Imc_0\subseteq\PIp$, and, assuming $\Imc_j\subseteq\PIp$ for some $j\geq 0$, then any atom in $\Imc_j$ can be derived by resolution from $\Phi$, thus, by construction, any atom in $\Imc_{j+1}\setminus\Imc_j$ can be derived from $\Phi$ by one additional resolution step, making them implicates of $\Phi$.
    Because they are atoms, they must be prime implicates, thus $\Imc_{j+1}\subseteq\PIp$, completing the induction.
    \qedhere
	\end{proof}
  Note that if $\Phi$ was a set of \emph{definite} Horn clauses, the above result would be immediate because it is well-known in logic programming \cite{Lloyd87}.
  The presence of the negative clause $\neg\negc{C_2}(\skz)\in\Phi$ is what justifies the existence of the current proof. 

  Now that $\PIp$ has been established as the universal Herbrand model of $\Phi$, the atoms it contains can be used to reconstruct concepts subsuming $C_1$ by means of a canonical model.
	\begin{lemma}[Canonical Model and $\PIp$]
    \label{lem:gcitopi}
		Given an abduction problem \mbox{$\langle\Tmc,\mathcal{H},C_1\isa C_2\rangle$}, its first-order translation $\Phi$ and an \EL-description tree $\Tr=(V,E,v_0,l)$, the entailment $\Tmc\models C_1\isa C_{\Tr}$ holds if and only if
		there exists a Skolem labeling $sl_{\Tr}$ such that $sl_{\Tr}(v_0)=\skz$ and $\CM(sl_{\Tr})\subseteq\PIp$.
	\end{lemma}
	\begin{proof}
    Given the preconditions of the lemma, let us first assume $\Tmc\models C_1\isa C_2$ to show the existence of a Skolem labeling $sl_{\Tr}$ such that $sl_{\Tr}(v_0)=\skz$ and $\CM(sl_{\Tr})\subseteq\PIp$.
		Since $C_1(\skz)\in\Phi$ by definition, we have $\skz\in (C_1)^\Imc$ for any Herbrand model $\Jmc$ of $\Phi$.
    Moreover, because $\Tmc\models C_1\isa C_{\Tr}$, it follows that  $\skz\in (C_\Tr)^\Imc$ for any Herbrand model $\Jmc$ of $\Phi$.
    By Lemma\ \ref{lem:pipmodel}, we know that $\PIp$ can be seen as a Herbrand model of $\Phi$, %
    thus $\skz\in (C_\Tr)^{\PIp}$.
    The existence of a Skolem labeling with the desired properties follows by Lemma \ref{lemma:ItoIc}.

    To prove the opposite implication, we assume given a Skolem labeling verifying $sl_{\Tr}(v_0)=\skz$ and $\CM(sl_{\Tr})\subseteq\PIp$ and show that $\Tmc \models C_1\isa C_\Tr$ by contradiction.
		Then $\skz\in (C_{\Tr})^{\CM(sl_{\Tr})}$ because $sl_{\Tr}(v_0)=\skz$.
		Towards contradiction, we assume $\Tmc\not\models C_1\isa C_{\Tr}$.
    Then $\pi(\Tmc)\not\models \pi(C_1\isa C_\Tr)$, since the standard translation from \EL to first-order logic preserves entailment \cite{DL_TEXTBOOK}.
    Thus $\pi(\Tmc)\land\neg \pi(C_1\isa C_\Tr)$ is satisfiable and hence, the Skolemizations of $$\pi(\Tmc)\land\neg \pi(C_1\isa C_\Tr)=\pi(\Tmc)\land\exists x.(C_1(x)\land \neg\pi(C_\Tr,x))$$ are also satisfiable.
    Let us consider the particular Skolemization $\varphi$ of $\pi(\Tmc)\wedge\neg\pi(C_1\isa C_\Tr)$ that coincides with $\Phi$ on the Skolemization of $\Tmc$ and uses $\skz$ to Skolemize the existential variable in $\neg\pi(C_1\isa C_\Tr)$.
    Let $\Imc'$ be a minimal Herbrand model of $\varphi$.
    It verifies $\skz\in (C_1)^{\Imc'}$ and $\skz\not\in (C_\Tr)^{\Imc'}$.
    We show that $\Imc'$ is a model of $\Phi$, which will allow us to raise a contradiction on that last statement.
    Since, by design, $\varphi$ contains all the non-renamed clauses in $\Phi$, it follows that $\Imc'$ satisfies these non-renamed clauses also for $\Phi$.
    Since $\varphi$ does not include renamed atoms, the minimality of $\Imc'$ ensures that it does not include any renamed atoms.
    This ensures that $\Imc'$ also models the renamed part of $\Phi$: for any renamed $\negc{C}\isa \negc{D}$, it holds that $(\negc{C})^{\Imc'}=(\negc{D})^{\Imc'}=\emptyset$, and $\neg \negc{C_2}(\skz)$ is also true in \Imc'.
    Thus, $\Imc'$ is a model of $\Phi$.
    However, since $\CM(sl_{\Tr})\subseteq \PIp$, and $\PIp\subseteq \Imc'$ by Lemma\ \ref{lem:pipmodel}, it follows that  $\CM(sl_{\Tr})\subseteq\Imc'$ must hold.
    In addition, since $C_\Tr(\skz)\in\CM(sl_{\Tr})$ because $sl_\Tr(v_0) = \skz$, it follows that $\skz\in (C_{\Tr})^{\Imc'}$, a contradiction.
    \qedhere
	\end{proof}

  Lemma\ \ref{lem:gcitopi} establishes a relation between the FOL encoding $\Phi$ and the original \EL problem, but we need a stronger result to know how to construct the $C_\Tr$ such that $\Tmc\models C_1\isa C_\Tr$ from $\PIp$.
  Lemma\ \ref{lem:posconnection} does the job, by showing that it is only necessary to collect the atomic prime implicates about unary predicates (the ones from $\NC$) to construct all relevant $C_\Tr$.

	\begin{lemma}[Construction of Subsumers of $C_1$]
    \label{lem:posconnection}
		Given an abduction problem \mbox{$\langle\Tmc,\mathcal{H},C_1\isa C_2\rangle$}, its first-order translation $\Phi$ and a set $\Amc=\{A_1(t_1),\ldots,$ $A_n(t_n)\}\subseteq \PIp$ where $n>0$, there exists $\Tr=(V,E,v_0,l)$ and $sl_{\Tr}$ s.t.\ $\Amc=\CMA(sl_{\Tr})$ and $\Tmc\models C_1\isa C_{\Tr}$.
	\end{lemma}
	\begin{proof}
    Given an abduction problem \mbox{$\langle\Tmc,\mathcal{H},C_1\isa C_2\rangle$}, its first-order translation $\Phi$ and a set $\Amc=\{A_1(t_1),\ldots,A_n(t_n)\}\subseteq \PIp$ where $n>0$, we notice that every singleton set $\{A_i(t_i)\}\subseteq\Amc$ also verifies that $\{A_i(t_i)\}\subseteq\PIp$.
    Thus to prove the property for any $\Amc$, we first show it for singletons and then we show how to construct a description tree for any $\Amc$ given the description trees for each singleton containing an element of $\Amc$.

    Let $\Amc = \{A(t)\}$ be a singleton.
    In practice, we need a slightly stronger property: we show the existence of a $\Tr=(V,E,v_0,l)$ and $sl_{\Tr}$ such that $\{A(t)\} = \CMA(sl_\Tr)$, $sl_\Tr(v_0) = \skz$ and $\CM(sl_\Tr)\subseteq\PIp$.
    Then $\Tmc\models C_1\isa C_\Tr$ follows by Lemma\ \ref{lem:gcitopi}.
    By Lemma\ \ref{lem:pipmodel}, $\PIp$ is a Herbrand model of $\Phi$ and thus $C_1(\skz)\in\PIp$. %
    As shown in the proof of Lemma\ \ref{lem:pipmodel}, we can write $\PIp$ as $\bigcup_{i\in\mathbb{N}}\Imc_i$, where:
    \begin{itemize}
    \item $\Imc_0=\{C_1(\skz)\}$ and,
    \item given $\Imc_j$,
      \begin{eqnarray*}
        \Imc_{j+1} &=& \Imc_j \cup \{B(t)\mid t\in (D)^{\Imc_j},\, \neg\pi(D,x)\vee B(x)\in\Phi\}\\
                   & &  \phantom{\Imc_j}\cup \{B(\sk(t)),r(t,\sk(t))\mid t\in (A)^{\Imc_j},\\
                   & & \qquad\quad\neg\pi(A,x)\vee B(\sk(x))\in \Phi,\, \neg\pi(A,x)\vee r(x,\sk(x))\in \Phi\}\\
                   & &  \phantom{\Imc_j}\cup \{B(t)\mid \sk(t)\in (A)^{\Imc_j},\, (t,\sk(t))\in r^{\Imc_j},\\
                   & & \qquad\quad\neg r(x,y)\vee\neg A(y)\vee B(x)\in\Phi\}.\\
      \end{eqnarray*}
    \end{itemize}
    Since $A(t)\in\PIp$, there exists an $i\in\mathbb{N}$ that is the smallest such that $A(t)\in\Imc_i$.
    We construct $\Tr$ inductively, depending on the value of $i$.
    If $i=0$, then $A(t)=C_1(\skz)$ and thus defining $\Tr$ and $sl_\Tr$ as $\Tr = (\{v_0\}, \emptyset, v_0, \{v_0\mapsto \{C_1\}\})$ and $sl_\Tr = \{v_0\mapsto \skz\}$ ensures additionally that $\Amc=\{C_1(\skz)\}=\CMA(sl_\Tr)=\CM(sl_\Tr)\subseteq\PIp$. %
    Assuming we know how to construct suitable description trees and Skolem labelings up to a given $j\in\mathbb{N}$, when $i=j+1$, the construction of $\Tr$ depends on the reason for which $A(t)\in\Imc_{j+1}\setminus\Imc_j$.
			\begin{itemize}
      \item If $A(t)\in\{B(t)\mid t\in (D)^{\Imc_j},\, \neg\pi(D,x)\vee B(x)\in\Phi\}$, where $D$ is in fact an atomic concept $B$ then $B(t)\in\Imc_{j}$ and thus $B(t)\in\PIp$.
        By induction, let $\Tr'=(V',E',v_0,l')$ and $sl_{\Tr'}$ be a description tree and Skolem labeling such that $sl_{\Tr'}(v_0)=\skz$, $\CM(sl_{\Tr'}) \subseteq\PIp$ and $\{B(t)\} = \CMA(sl_{\Tr'})$.
        Let $v\in V'$ be the node such that $l_{\Tr'}(v) = \{B\}$ and $sl_{\Tr'}(v) = t$.
        Then the Skolem labeling $sl_\Tr$ is defined as identical to $sl_{\Tr'}$ and we define $\Tr$ as $(V',E',v_0,l'[v\mapsto \{A\}])$, where $l'[v\mapsto \{A\}]$ denotes the function $l'$ except on $v$ for which the value returned is $\{A\}$ so that $\{A(t)\} = \CMA(sl_\Tr)$ as wanted.
        Since $\CM(sl_\Tr) \subseteq \{A(t)\}\cup\CM(sl_{\Tr'})$ and $A(t)\in\PIp$, it follows that $\CM(sl_\Tr)\subseteq\PIp$.
      \item If $A(t)\in\{B(t)\mid t\in (D)^{\Imc_j},\, \neg\pi(D,x)\vee B(x)\in\Phi\}$, where $D$ is in fact the conjunction of two atomic concepts $B_1$ and $B_2$, then both $B_1(t)$ and $B_2(t)$ belong to $\Imc_j$ and thus to $\PIp$.
        We adapt exactly as in the last case any of the description trees $\Tr_1$ or $\Tr_2$ and associated Skolem labeling $sl_{\Tr_1}$ or $sl_{\Tr_2}$, that respectively correspond to $B_1$ and $B_2$ and verify the properties by induction. %
      \item If $A(t)\in \{B(\sk(t))\mid t\in (A')^{\Imc_j}, r(t,\sk(t))\in\Imc_{j+1}, \neg\pi(A',x)\vee B(\sk(x))\in \Phi, \\ \neg\pi(A',x)\vee r(x,\sk(x))\in \Phi\}$ then $t = \sk(t')$ for some $\sk$ and $t'$ such that $A'(t')\in\Imc_j$, $\neg A'(x)\vee A(\sk(x)), \neg A'(x)\vee r(x,\sk(x))\in\Phi$ for some $r$ and $A'$.
        Since $A'(t')\in\PIp$, there exists a description tree $\Tr'=(V',E',v_0,l')$ such that $sl_{\Tr'}(v_0) = \skz$, $\CMA(sl_{\Tr'}) = \{A'(t')\}$, and $\CM(sl_{\Tr'})\isa\PIp$.
        Let $v'$ be the leaf node such that $l'(v) = \{A'\}$ and $sl_{\Tr'}(v) = t'$.
        We introduce a fresh node $v$ to define $\Tr$ as $(V'\cup\{v\},\, E'\cup \{v'rv\},\, v_0,\, l'[v'\mapsto\emptyset]\cup\{v\mapsto\{A\}\})$ and $sl_\Tr=sl_{\Tr'}\cup\{v\mapsto t\}$.
        Thus $sl_\Tr(v_0)=sl_{\Tr'}(v_0) = \skz$, $\CMA(sl_\Tr) = \{A(t)\}$ and $\CM(sl_\Tr) \subseteq \{A(\sk(t)), r(t,\sk(t))\}\cup \CM(sl_{\Tr'})\subseteq\PIp$ since $\{A(\sk(t')), r(t',\sk(t'))\}\subseteq\PIp$.
      \item If $A(t)\in \{B(t)\mid \sk(t)\in (A')^{\Imc_j},\, (t,\sk(t))\in r^{\Imc_j},\, \neg r(x,y)\vee\neg A'(y)\vee \\ B(x)\in\Phi\}$ then there exist $A'$, $r$ and $\sk$ such that $A'(\sk(t))\in \Imc_j$, $r(t,\sk(t))\in \Imc_j$, and $\neg r(x,y)\vee\neg A'(y)\vee A(x)\in\Phi$.
        By induction, we consider a description tree $\Tr'=(V',E',v_0,l')$ and associated Skolem labeling $sl_{\Tr'}$ for which $sl_{\Tr'}(v_0) = \skz$, $\CMA(sl_{\Tr'}) = \{A'(t')\}$, and $\CM(sl_{\Tr'}) \subseteq\PIp$.
        Let $v$ be the leaf in $V'$ such that $l'(v) = \{A'\}$ and $w$ be its parent in the tree, such that $wr'v\in E'$ for some $r'$.
        We define $\Tr$ as $(V'\setminus\{v\}, E'\setminus\{wr'v\}, v_0, l'[w\mapsto\{A\}]\setminus \{v\mapsto\{A'\}\})$ and $sl_\Tr = sl_{\Tr'}\setminus\{v\mapsto \sk(t)\}$.
        Thus, $sl_\Tr(v_0)=sl_{\Tr'}(v_0)=\skz$, $\CMA(sl_\Tr) = \{A(t)\}$ and $\CM(sl_\Tr) \subseteq \{A(t)\}\cup\CM(sl_{\Tr'})\subseteq\PIp$.
			\end{itemize}

      Let us now consider the case of non-singleton $\Amc = \{A_1(t_1),\dots,A_n(t_n)\}$ ($n>1$).
      We have just seen how to obtain description trees $\Tr_i=(V_{i},E_{i},v_0,l_{i})$ and Skolem labelings  $sl_{\Tr_i}$ for $i\in\{1,\dots,n\}$ such that $\{A_i(t_i)\} = \CMA(sl_{\Tr_i})$, $sl_{\Tr_i}(v_0) = \skz$ and $\CM(sl_{\Tr_i})\subseteq\PIp$.
      We define $\Tr = (V,E,v_0,l)$ and $sl_\Tr$ by introducing a node $v\in V$ for each $t\in\bigcup_{i=1}^n\{sl_{\Tr_i}(v')\mid v'\in V_i\}$ and setting $sl_\Tr(v) = t$ in each case.
      For $t=\skz$, the introduced node $v\in V$ is named $v_0$ and declared as the root of $\Tr$.
      It remains to define $E$ and $l$.
      For $E$ we collect all edges from the description trees $\Tr_i$ to obtain
      $$E=\bigcup_{i=1}^n\{vrw\mid v'rw'\in E_i, sl_{\Tr_i}(v') = sl_{\Tr}(v), sl_{\Tr_i}(w')=sl_\Tr(w)\}.$$
      For $l$, we proceed similarly to collect labels, producing for each $v\in V$,
      $$l(v)=\bigcup_{i=1}^n\{l_i(v')\mid v'\in V_i, sl_{\Tr_i}(v')=sl_\Tr(v)\}.$$
      Thus
      \begin{align*}
        \CMA(sl_\Tr) &= \{A(sl_\Tr(v))\mid A\in l(v), v\in V\}\\
        &= \bigcup_{i=1}^n\{A(sl_{\Tr_i}(v))\mid A\in l_i(v'), v'\in V_i\}\\
        &= \bigcup_{i=1}^n\CMA(sl_{\Tr_i}) = \Amc
      \end{align*}
      and since also
      \begin{align*}
        \CMr(sl_\Tr) &= \{r(sl_\Tr(v),sl_\Tr(w))\mid vrw\in E\}\\
                     &= \bigcup_{i=1}^n\{r(sl_{\Tr_i}(v'),sl_{\Tr_i}(w'))\mid v'rw'\in E_i, sl_{\Tr_i}(v') = sl_\Tr(v), sl_{\Tr_i}(w') = sl_\Tr(w)\}\\
        &= \bigcup_{i=1}^n\CMr(sl_{\Tr_i}),
      \end{align*}
      we have $\CM(sl_\Tr) = \bigcup_{i=1}^n\CM(sl_{\Tr_i})\subseteq\PIp$, and by Lemma\ \ref{lem:gcitopi}, $\Tmc\models C_1\isa C_\Tr$.
    \qedhere
	\end{proof} 

  \subsection{Subsumees of $C_2$ and Negative Prime Implicates}
	Next, we show how negative ground implicates are related to the solutions of the abduction problem.
	\begin{lemma}[Concept Homomorphism and Negative Implicates]
    \label{lem:homtonegimp}
    Let \mbox{$\langle\Tmc,\sig,C_1\isa C_2\rangle$} be an abduction problem, $\Phi$ denote its translation to first-order logic and $\Amc = \{A_1(t_1),\ldots,A_k(t_k)\}\subseteq \PIp$.
    As allowed by Lemma\ \ref{lem:posconnection}, let $\Tr_1=(V_1,E_1,v_0,l_1)$ and $sl_{\Tr_1}$ denote an \EL-description tree and associated Skolem labeling s.t.\ $sl_{\Tr_1}(v_0)=\skz$, $\Amc=\CMA(sl_{\Tr_1})$ and $\Tmc\models C_1\isa C_{\Tr_1}$.

    For any \EL-description tree $\Tr_2=(V_2,E_2, w_0, l_2)$ with a weak homomorphism $\phi$ from $\Tr_2$ to $\Tr_1$, the following equivalence holds:
		\begin{description}
			\item[(EL)] $\Tmc\models C_{\Tr_2}\isa C_2$
		\end{description}
    if and only if 
			\begin{description}
      \item[(FO)] there is a Skolem labeling $sl_{\Tr_2}$ for $\Tr_2$ such that
        \begin{itemize}
        \item[] $sl_{\Tr_2}(v) = sl_{\Tr_1}(\phi(v))$ for all $v\in V_2$, and
        \item[] $\Phi\models \bigvee_{v\in V_2,B\in l_{\Tr_2}(v)}\neg \negc{B}(sl_{\Tr_2}(v))$.
        \end{itemize}
			\end{description}
			
	\end{lemma}
	\begin{proof}
    
    We first show that \textbf{(EL)} implies \textbf{(FO)}.
    We thus assume that $\Tmc\models C_{\Tr_2}\isa C_2$.
    We define $sl_{\Tr_2}$ as $sl_{\Tr_2}(v)=sl_{\Tr_1}(\phi(v))$ for all $v\in V_{\Tr_2}$.
    Since $\phi$ is a weak homomorphism from $\Tr_2$ to $\Tr_1$, $sl_{\Tr_2}(v)$ is a Skolem labeling for $C_{\Tr_2}$.
    It remains only to show that $\Phi\models \bigvee_{v\in V_2,B\in l_{\Tr_2}(v)}\neg \negc{B}(sl_{\Tr_2}(v))$.

    By assumption $\Tmc\models C_{\Tr_2}\isa C_2$, thus $\pi(\Tmc)\models \neg\pi(C_{\Tr_2},x)\vee \pi(C_2,x)$ by a direct translation, from which we deduce $\pi(\Tmc)\wedge \neg\pi(C_2,\skz)\models \neg\pi(C_{\Tr_2},\skz)$.
    This entailment also holds for the renamed versions of $\Tmc$, $C_{\Tr_2}$ and $C_2$ so $\Phi\models\neg\pi(\negc{C_{\Tr_2}},\skz)$.
    The clause $\neg\pi(\negc{C_{\Tr_2}},\skz)$ has the following form:
    $$\neg\pi(\negc{C_{\Tr_2}},\skz)=\bigvee_{urv\in E_2}\neg r(t_u,t_v)\lor\bigvee_{v\in V_2,\,B\in l_2(v)}\neg \negc{B}(t_v)$$
    where, $t_{w_0}=\skz$ and for all $v\in V_{\Tr_2}\setminus\{w_0\}$, $t_v$ is a variable uniquely associated with $v$.

    Since $\CM(sl_{\Tr_1})\subseteq\PIp$, in particular $\CMr(sl_{\Tr_1})\subseteq\PIp$.
    Moreover, for each edge $urv\in E_{\Tr_2}$, $\Phi\models r(sl_{\Tr_2(u)},sl_{\Tr_2(v)})$ and thus $r(sl_{\Tr_2(u)},sl_{\Tr_2(v)})\in\CMr(sl_{\Tr_1})$ can be derived from $\Phi$ using the resolution calculus.
    Each of these atomic clauses can be resolved away with the corresponding literal $\neg r(t_u,t_v)$ in $\neg\pi(\negc{C_{\Tr_2}},\skz)$.
    In this derivation, all $t_v$ variables are replaced with $sl_{\Tr_2}(v)$.
    In addition $\skz=sl_{\Tr_2}(w_0)$, thus the clause $\bigvee_{v\in V_2,\,B\in l_2(v)}\neg \negc{B}(sl_{\Tr_2}(v))$ is derivable from $\Phi$ by resolution, and thus $\Phi \models \bigvee_{v\in V_2,\,B\in l_2(v)}\neg \negc{B}(sl_{\Tr_2}(v))$, so \textbf{(FO)} holds.
	
    Let us now assume \textbf{(FO)} in order to prove \textbf{(EL)}.
    We consider a Herbrand interpretation $\Imc=\PIp\cup \bigcup_i \Imc_i$ where $\Imc_i$ for $i\in\mathbb{N}$ is defined inductively as:
\begin{itemize}
	\item $\Imc_0=\{\negc{B}(t)\mid B(t)\in\CMA(sl_{\Tr_2})\}\cup\CMr(sl_{\Tr_2})$ and,
	\item given $\Imc_j$,
    \begin{eqnarray*}
      \Imc_{j+1} &=& \Imc_j \cup \{\negc{B}(t)\mid t\in (\negc{D})^{\Imc_j},\, \neg\pi(\negc{D},x)\vee \negc{B}(x)\in\Phi\}\\
                 & &  \phantom{\Imc_j}\cup \{\negc{B}(\sk(t)),r(t,\sk(t))\mid t\in (\negc{A})^{\Imc_j},\\
                 & & \qquad\quad\neg\pi(\negc{A},x)\vee \negc{B}(\sk(x))\in \Phi,\, \neg\pi(\negc{A},x)\vee r(x,\sk(x))\in \Phi\}\\
                 & &  \phantom{\Imc_j}\cup \{\negc{B}(t)\mid \sk(t)\in (\negc{A})^{\Imc_j},\, (t,\sk(t))\in r^{\Imc_j},\\
                 & & \qquad\quad\neg r(x,y)\vee\neg \negc{A}(y)\vee \negc{B}(x)\in\Phi\}.\\
    \end{eqnarray*}
\end{itemize}

The $\Imc_i$s are built to collect all the elements necessary to make the renamed part of $\Phi$ true, one step at a time, starting from an interpretation that satisfies $\bigwedge_{v\in V_{\Tr_2},B\in l_{\Tr_2}(v)} \negc{B}(sl_{\Tr_2}(v))\subseteq \CM(sl_{\Tr_2})$, and is thus incompatible with $\Phi$ under the \textbf{(FO)} assumption.
Indeed since $\Tmc$, and by extension $\negc{\Tmc}$, is in normal form, it contains only concept inclusions of the form $D\isa B$, $A\isa \exists r. B$ and $\exists r. A\isa B$, where $D$ is either an atomic concept or a conjunction of two atomic concepts.
These correspond in $\Phi$ respectively to the clauses $\neg\pi(\negc{D},x)\vee \negc{B}(x)$, to the pair of clauses $\{\neg\pi(\negc{D},x)\vee \negc{B}(\sk(x)),\, \neg\pi(\negc{D},x)\vee r(x,\sk(x))\}$ and to the clause $\neg r(x,\sk(x))\vee\neg \negc{A}(\sk(x))\vee \negc{B}(x)$.
Hence, if $t\in (\negc{D})^\Imc$ (resp.\ $\sk(t)\in (\negc{A})^\Imc$ and $(t,\sk(t))\in r^\Imc$) then there exists some $i\in\mathbb{N}$ such that $t\in (\negc{D})^{\Imc_i}$ (resp.\ $\sk(t)\in (\negc{A})^{\Imc_i}$ and $(t,\sk(t))\in r^{\Imc_i}$) and then all concept inclusions where $\negc{D}$ occurs on the right-hand side (resp. where $\exists r.\negc{A}$ occurs on the right-hand side for some $r$) are satisfied in $\Imc_{i+1}$.

Since $\Imc$ includes $\PIp$, the satisfiability of $C_1(\skz)$ and the Skolemization of $\pi(\Tmc)$ can be shown as in Lemma\ \ref{lem:posconnection}.
Thus, given that $\Imc\models \Phi\setminus\{\neg \negc{C_2}(\skz)\}$ but $\Imc\not\models \Phi$, necessarily $\Imc\models \negc{C_2}(\skz)$.

To make use of that fact, we must first prove the following statement:
\begin{itemize}
\item[($*$)] For any set $\Bmc=\{\negc{B}_1(t_1),\ldots,\negc{B}_k(t_k)\}\subseteq \Imc$, there exists $\Tr=(V,E,v_0,l)$ and $sl_{\Tr}$ s.t.\ $\Bmc=\{\negc{B}(sl_{\Tr}(v))\mid B\in l(v),\, v\in V\}$ and $\Tmc\models C_{\Tr_2}\isa C_{\Tr}$.
\end{itemize}
Without loss of generality, we can consider the biggest such $\Bmc$, that is the set of all atoms of the form $\negc{A}(t)$ in $\Imc$.
The smaller $\Bmc$s simply correspond to concepts $C_\Tr$ with fewer conjuncts.

Since $\PIp$ only contains non-renamed concepts, we prove ($*$) by induction on the $\Imc_j$ for $j\in\mathbb{N}$.
		When $j=0$, $\Imc_0=\CM(sl_{\Tr_2})$, thus $C_T=C_{\Tr_2}$ and the result directly follows.
		Assuming the result holds for a given $\Imc_j$, let $\Bmc=\{\negc{B}_1(t_1),\ldots,\negc{B}_k(t_k)\mid \negc{B}_i(t_i)\in\Imc_{j+1}\}$, and let $\Bmc^* = \{\negc{B}_1(t_1),\ldots,\negc{B}_l(t_l)\mid \negc{B}_i(t_i)\in\Imc_{j}\}$.
      The induction hypothesis applies to $\Bmc^*$ and we conclude that there exists an \EL-description tree $\Tr^*=(V_*,E_*,v_0,l_*)$ and a Skolem labeling $sl_{\Tr^*}$ s.t.\ $\Tmc\models C_{\Tr_2}\isa C_{\Tr^*}$ and $\Bmc^*= \{\negc{B}(sl_{\Tr}(v))\mid B\in l_*(v),\, v\in V_*\}$.
      Let us now consider the literals in $\Bmc\setminus \Bmc^*$.
      They are all of the form $\negc{B}(t)$ and belong to $\Imc_{j+1}$.
      We define \Tr and $sl_{\Tr}$ by extending ${\Tr^*}$ and $sl_{\Tr^*}$.
      The extension for each $\negc{B}(t)$ depends of which set it originates from.
		\begin{itemize}
    \item If $\negc{B}(t)\in \{\negc{B}(t)\mid t\in (\negc{D})^{\Imc_j},\, \neg\pi(\negc{D},x)\vee \negc{B}(x)\in\Phi\}\setminus \Imc_j$, let $v$ be the node in $\Tr^*$ such that $l_*(v)$ contains all atomic concepts from $D$ and $sl_{\Tr^*}(v) = t$.
      We add $B$ to $l_*(v)$ and the rest of $\Tr^*$ and $sl_{\Tr^*}$ is unchanged.
      Note that, in that case, $D\isa B\in\Tmc$ by construction of $\Phi$.
    \item If $\negc{B}(t)\in \{\negc{B}(t)\mid \sk(t)\in (\negc{A})^{\Imc_j},\, (t,\sk(t))\in r^{\Imc_j},\, \neg r(x,y)\vee\neg \negc{A}(y)\vee \negc{B}(x)\in\Phi\}\setminus \Imc_j$, let $v$ be the node in $\Tr^*$ such that $vrw\in E_*$, $sl_{\Tr^*}(v)=t$, $sl_{\Tr^*}(v)=\sk(t)$, and  $A\in l_*(w)$.
      As in the previous case, we simply add $B$ to $l_*(v)$.
      In that case, $\exists r. A\isa B\in \Tmc$ for the corresponding $r$.
    \item If $\negc{B}(\sk(t))\in \{\negc{B}(\sk(t)),r(t,\sk(t))\mid t\in (\negc{D})^{\Imc_j},\,\neg\pi(\negc{D},x)\vee \negc{B}(\sk(x))\in \Phi,\, \neg\pi(\negc{D},x)\vee r(x,\sk(x))\in \Phi\}\setminus \Imc_j$, let $v$ be the node such that $sl_{\Tr^*}(v) = t$.
      Then we add a fresh node $w$ to $V_*$ as well as an edge $vrw$ to $E_*$.
      We also extend $sl_{\Tr^*}$ so that $w$ is mapped to $t$.
      In that case, $A\isa \exists r.B\in \Tmc$ for the corresponding $r$.
		\end{itemize}
    Note that in all cases, $\Tmc\models C_{\Tr^*}\isa C_{\Tr}$ because the conjunct(s) added from $C_{\Tr^*}$ to $C_{\Tr}$ is(/are) justified by the concept inclusion from $\Tmc$ that is ultimately to blame for the existence of $\negc{B}(t)$ in $\Imc_{j+1}\setminus \Imc_j$.
    Since, by the induction hypothesis, $\Tmc\models C_{\Tr_2}\isa C_{\Tr^*}$, \Tr and $sl_\Tr$ are s.t.\ $\Bmc=\{\negc{B}(sl_{\Tr}(v))\mid B\in l(v),\, v\in V\}$ and $\Tmc\models C_{\Tr_2}\isa C_{\Tr}$, ($*$) holds for that case and thus also for $\Imc$.

    Because $\Imc\models \negc{C_2}(\skz)$, necessarily $\negc{C_2}(\skz)\in\Imc$.
    Thus, by ($*$), $\Tmc\models C_{\Tr_2}\isa C_2$ because the \EL-description tree \Tr from ($*$) in that case is such that $C_\Tr = C_2$.
    \qedhere
	\end{proof}

	\begin{lemma}[\EL-Description Tree and $\smand$]
		\label{lemma:subsethom}
    Given two \EL-description trees $\Tr_1=(V_1, E_1, v_0, l_1)$ and $\Tr_2=(V_2, E_2, w_0, l_2)$,
		$C_{\Tr_1}\smand C_{\Tr_2}$ if and only if there is an (injective) $\emptyset$-homomorphism from $\Tr_1$ to $\Tr_2$.
	\end{lemma}
	\begin{proof}
		If $C_{\Tr_1}\equiv C_{\Tr_2}$, we are done because then $\Tr_1$ and $\Tr_2$ obviously have the same shape.
    Otherwise, the missing conjuncts in $C_1$ would correspond to either:
    \begin{itemize}
    \item some missing atomic concepts in some $l_2(w)$ from $\Tr_2$ or
    \item a subtree of $\Tr_2$ that is not in the image of the $\emptyset$-homomorphism from $\Tr_1$.
      \qedhere
    \end{itemize}
	\end{proof}

	Last, we show how prime implicates are related to the connection-minimal solutions of the abduction problem.
	\begin{lemma}[Concept Homomorphism and Negative Implicates]
		\label{lemma:matchingpiandhom}
    Let \mbox{$\langle\Tmc,\sig,C_1\isa C_2\rangle$} be an abduction problem, $\Phi$ its translation to first-order logic, and $\Amc = \{A_1(t_1),\ldots,A_k(t_k)\}\subseteq \PIp$.
    As allowed by Lemma\ \ref{lem:posconnection}, let $\Tr_1=(V_1,E_1,v_0,l_1)$ and $sl_{\Tr_1}$ denote an \EL-description tree and associated Skolem labeling s.t.\ $\Amc=\{A(sl_{\Tr_1}(v))\mid A\in l_1(v), v\in V_1\}$ and $\Tmc\models C_1\isa C_{\Tr_1}$.

    For any \EL-description tree $\Tr_2=(V_2,E_2, w_0, l_2)$ with a weak homomorphism $\phi$ from $\Tr_2$ to $\Tr_1$, the following equivalence holds:
		\begin{itemize}
			\item[] $C_{\Tr_2}$ is a $\smand$-minimal concept s.t. $\Tmc\models C_{\Tr_2}\isa C_2$
		\end{itemize}
    if and only if 
			\begin{itemize}
      \item[] there is a Skolem labeling $sl_{\Tr_2}$ for $\Tr_2$ s.t.\
        \begin{itemize}
        \item[] $sl_{\Tr_2}(v) = sl_{\Tr_1}(\phi(v))$ for all $v\in V_{\Tr_2}$, and
        \item[] $\bigvee_{v\in V_2,B\in l_2(v)}\neg \negc{B}(sl_{\Tr_2}(v))\in\PIn$.
        \end{itemize}
			\end{itemize}
			
	\end{lemma}
  \begin{proof}
    Thanks to Lemma\ \ref{lem:homtonegimp}, we know that, in the conditions of the lemma, the existence of a $C_{\Tr_2}$ such that $\Tmc\models C_{\Tr_2}\isa C_2$ is equivalent to the existence of a Skolem labeling $sl_{\Tr_2}$ for $\Tr_2$ such that $sl_{\Tr_2}(v) = sl_{\Tr_1}(\phi(v))$ for all $v\in V_{\Tr_2}$, and $\Phi\models\bigvee_{v\in V_2,B\in l_2(v)}\neg \negc{B}(sl_{\Tr_2}(v))$.
    It remains to show the equivalence between the $\smand$-minimality of $C_{\Tr_2}$ and the fact that $\bigvee_{v\in V_2,B\in l_2(v)}\neg \negc{B}(sl_{\Tr_2}(v))\in\PIn$.

    By Lemma\ \ref{lemma:subsethom}, for any two trees $\Tr_2'$ and $\Tr_2''$ and corresponding Skolem labelings $sl_{\Tr_2'}$ and $sl_{\Tr_2''}$ for which there are respective weak homomorphisms $\phi_1$ and $\phi_2$ to $\Tr_1$,
    $C_{\Tr_2'}\smand C_{\Tr_2''}$ if and only if $ \{\negc{B}(sl_{\Tr_2'}(v))\mid v\in V_{2}',B\in l_{2}'(v)\}\subseteq \{ \negc{B}(sl_{\Tr_2''}(v))\mid v\in V_{2}'',B\in l_{2}''(v)\}$.
    Thus it is not possible for $C_{\Tr_2}$ to be $\smand$-minimal if $\bigvee_{v\in V_2,B\in l_2(v)}\neg \negc{B}(sl_{\Tr_2}(v))$ is not prime and vice-versa.
    \qedhere
  \end{proof}

  Theorem\ \ref{corr:constr}, that shows how to construct solutions for an abduction problem from prime implicates of its translation to first-order logic, is a consequence of Lemma \ref{lemma:matchingpiandhom}.
  \Construction*
  \begin{proof}
    Let $\langle \Tmc,\Sigma, C_1\isa C_2 \rangle$ be an abduction problem and $\Phi$ be its first-order translation.
    
    We begin by assuming given a packed connection-minimal hypothesis $\Hmc$.
    Then there exist concepts $D_1$ and $D_2$, and weak homomorphism $\phi$ verifying points\ \ref{it:reach}-\ref{it:hom} of Def.\ \ref{def:abd} while \Hmc verifies point \ref{it:H} of the same definition for these $D_1$, $D_2$ and $\phi$.
    W.l.o.g., we consider that $D_1$ is such that every node in $\Tr_{D_1}$ is in the range of $\phi$.
    Such a $D_1$ can always be obtained from a $D_1$ that has too much nodes by pruning the extra nodes, since they cannot have children that are in the range of $\phi$.
    Since, by Def.\ \ref{def:abd} point\ \ref{it:reach}, $\Tmc\models C_1\isa D_1$, by Lemma\ \ref{lem:gcitopi} there exists a Skolem labeling $sl_1$ for $\Tr_{D_1} = (V_1, E_1, v_0, l_1)$ s.t.\ $\CM(sl_1)\subseteq\PIp$.
    Since $\Hmc$ is packed, $l_1(v)$ is also maximal for each $v$, so that $\CMA(sl_1) = \{A(sl_1(v))\in\PIp\mid v\in V_1\}$.
    Note that $\CMA(sl_1)$ cannot be empty and that there are no two nodes in $\Tr_{D_1}$ with the same Skolem label, otherwise $\CM(sl_1)\subseteq\PIp$ would not hold since this would imply that two occurrences of Skolem terms in $\Phi$ share the same Skolem function, which is forbidden in the standard Skolemization procedure.
    From point\ \ref{it:hom} of Def.\ \ref{def:abd}, we know that $\phi$ is a weak homomorphism from $\Tr_{D_2}$ to $\Tr_{D_1}$ and from point\ \ref{it:reachbw} that $D_2$ is a \smand-minimal concept s.t.\ $\Tmc\models D_2\isa C_2$.
    Hence, by Lemma\ \ref{lemma:matchingpiandhom}, there also exists a Skolem labeling $sl_2$ for $\Tr_{D_2}=(V_2,E_2,w_0,l_2)$ s.t.\ $sl_2(v) = sl_1(\phi(v))$ for all $v\in V_2$ and $\bigvee_{v\in V_2, B\in l_2(v)}\neg\negc{B}(sl_2(v))\in\PIn$.
    Since $\{B(sl_2(v))\mid v\in V_2, B\in l_2(v)\} = \CMA(sl_2)$, we define $\Bmc$ as $\CMA(sl_2)$.
    Our choice of $D_1$ allows us to define $\Amc$ as $\CMA(sl_1)$ since it holds that $sl_2(v) = sl_1(\phi(v))$ and there are no nodes in $V_1$ outside the range of $\phi$.
    Thus $\Amc$ and $\Bmc$ verify the first two points of Th.\ \ref{corr:constr}.
    Let us now consider any concept inclusion in \Hmc.
    It is of the form $\bigsqcap l_1(\phi(w))\isa\bigsqcap l_2(w)$ for some $w\in V_2$ and s.t.\ $\Tmc\not\models \bigsqcap l_1(\phi(w))\isa\bigsqcap l_2(w)$.
    For every $v\in V_1$, consider all $w_1,\ldots,w_k\in V_2$ s.t.\ $\phi(w_1)=\ldots=\phi(w_k)=v$.
    Then, $\Hmc$ contains $\{\bigsqcap l_1(v)\isa\bigsqcap l_2(w_1),\ldots,\bigsqcap l_1(v)\isa\bigsqcap l_2(w_k)\}$ which is equivalent to $\{\bigsqcap l_1(v)\isa(\bigsqcap l_2(w_1))\sqcap\ldots\sqcap (\bigsqcap l_2(w_1))\} = \{C_{\Amc,sl_1(v)}\isa C_{\Bmc,sl_1(v)}\}$ and since $\Tmc\not\models \bigsqcap l_1(v)\isa\bigsqcap l_2(w_i)$ for all $i\in\{1,\ldots,k\}$, also $\Tmc\not\models C_{\Amc,t}\isa C_{\Bmc,t}$ for $t = sl_1(v)$.
    This means in particular that this CI is not a tautology, ensuring that $C_{\Bmc,t}\not\smand C_{\Amc,t}$.
    Thus $\Hmc$ is equivalent to the constructible hypothesis built for $\Amc$ and $\Bmc$ as just defined.

    Now, let us consider that $\Hmc$ is a constructible hypothesis obtained from a given \Amc and \Bmc verifying the constraints from Th.\ \ref{corr:constr}.
    Then \Amc is a subset of $\PIp$, thus, by Lemma \ref{lem:posconnection}, there is a description tree $\Tr_1=(V,E,v_0,l_1)$ and associated Skolem labeling $sl$ s.t.\ $\Amc = \CMA(sl)$ and $\Tmc\models C_1\sqsubseteq C_{\Tr_1}$.
    We define $\Tr_2 = (V,E,v_0,l_2)$, where for all $v \in V$, $l_2(v)=\{B\mid B(sl(v))\in\Bmc\}$ and $\phi$ as the identity over $V$.
    Then $\phi$ is a weak homomorphism from $\Tr_2$ to $\Tr_1$ and $sl$ can also be associated to $\Tr_2$ and it is such that $\left(\bigvee_{v\in V, B\in l_2(v)}\neg\negc{B}(sl(v))\right)\in\PIn$.
    Thus, by Lemma \ref{lemma:matchingpiandhom}, $C_{\Tr_2}$ is a \smand-minimal concept s.t.\ $\Tmc\models C_{\Tr_2}\isa C_2$.
    As seen in the first part of this proof, $sl$ must be injective on $V$ due to its association with $\Tr_1$, and thus,  for all $v\in V$, $C_{\Amc,sl(v)}=\bigsqcap l_1(v)$ and by construction of $\Tr_2$, $C_{\Bmc,sl(v)} = \bigsqcap l_2(v)$.
    If $\Tmc\models\bigsqcap l_1(v)\isa\bigsqcap l_2(v)$, then $\Tmc\models C_{\Amc,sl(v)}\isa C_{\Bmc,sl(v)}$.
    We show that this implies $C_{\Bmc,sl(v)}\smand C_{\Amc,sl(v)}$.
    Consider any $t$ s.t.\ $\Tmc\models C_{\Amc,t}\isa C_{\Bmc,t}$.
    Then by translation, it means that $\Phi\models \neg\pi(C_{\Amc,t})\vee\pi(C_{\Bmc,t})$ and since both concepts do not contain role restrictions, it means in particular that $\Phi\models \bigvee_{A\in C_{\Amc,t}}\neg A(x)\vee B(x)$ for all $B\in C_{\Bmc,t}$.
    Since $A(t)\in\PIp$ for all $A\in C_{\Amc,t}$, $\Phi\models B(t)$ for all $B\in C_{\Bmc,t}$ and since those are atomic ground positive implicates, for all $B\in C_{\Bmc,t}$, $B(t)\in\PIp$.
    Furthermore, by definition of $\Amc$ and $C_{\Amc,t}$, this leads to $B\in C_{\Amc,t}$ for all $B\in C_{\Bmc,t}$, hence $C_{\Bmc,t}\smand C_{\Amc,t}$.
    In the particular case that interests us, it means that $C_{\Bmc,sl(v)}\smand C_{\Amc,sl(v)}$ as wanted.
    Hence $\Hmc$ is a connection-minimal hypothesis.
    It remains only to show that it is packed.
    Any tree $\Tr'$ built from $\Tr_1$ by extending the label of some $v\in V$ must be such that $\CM(sl_{\Tr'})\not\subseteq\PIp$, where $sl_{\Tr'}$ is identical to $sl$ but associated to $\Tr'$, since the labels of $\Tr_1$ are already maximal in that regard.
    Thus, by Lemma \ref{lem:gcitopi}, $\Tmc\not\models C_1\isa C_{\Tr'}$, hence any such $C_{\Tr'}$ cannot be used to create constructible hypotheses, proving $\Hmc$ packed.
	\qedhere
  \end{proof}
  
  As an illustration, consider 
  \begin{eqnarray*}
  \Tmc=\{&C_1\isa \exists r_1.(A\dlAnd B),& \\
  &\exists r_1.C\dlAnd\exists r_1.D\isa C_2&\}.
  \end{eqnarray*}
  The negative prime implicate $\neg \negc{C}(sk_1(sk_0))\lor\neg \negc{D}(sk_1(sk_0))$ corresponds to a tree and associated Skolem labeling as follows:
  $$\Tr_2=(\{v_0,v_1,v_2\}, \{v_0r_1v_1,v_0r_1v_2\},v_0,l_2)$$
  s.t.\ $l_2(v_0)=\emptyset$, $l_2(v_1)=\{C\}$ and $l_2(v_2)=\{D\}$; and $sl_{\Tr_2}(v_0)=\skz$, $sl_{\Tr_2}(v_1)=sl_{\Tr_2}(v_2)=\sk_1(\skz)$ . For $C_{\Tr_2}=\exists r_1.C\dlAnd\exists r_1.D$, and $C_{\Tr_1}=\exists r_1.(A\dlAnd B)$ the set $\{A\dlAnd B\isa C,A\dlAnd B\isa D\}$ is a packed connection-minimal hypothesis and the equivalent constructible hypothesis $\{A\dlAnd B\isa C\dlAnd D\}$ is the one found by applying Th.\ \ref{corr:constr}.

\section{Termination}
\label{app:term}
The proofs of Theorem\ \ref{thm:variable-bound} and\ \ref{the:bound} are detailed in this appendix.

  We first recall the notions used to describe the resolution calculus.
  A substitution is a function mapping variables to terms.
  The result of \emph{applying a substitution} $\sigma$ on a clause $\varphi$ is denoted by $\varphi\sigma$,
  and is the clause obtained by replacing every variable $x$ in $\varphi$ by $\sigma(x)$. A \emph{most general unifier (mgu)} of the atoms $P(\overline{t})$ and $P(\overline{t'})$, is a substitution s.t.\ $P(\overline{t})\sigma=P(\overline{t'})\sigma$, and for any other such substitution $\sigma'$, there exists a
  substitution $\sigma''$ so that $\sigma'=\sigma\circ\sigma''$.
The resolution calculus is made of two rules: resolution and factorization.
  The resolution rule infers from two \emph{premises} of the form $\varphi\vee P(\overline{t})$ and $\varphi'\vee \neg P(\overline{t'})$ the \emph{resolvent} $(\varphi\vee \varphi')\sigma$, given that an mgu $\sigma$ exists for $P(\overline{t})$ and $P(\overline{t'})$.
  The factorization rule infers from a premise of the form $\varphi\vee P(t)\vee P(t')$ the resolvent $(\varphi\vee P(t))\sigma$ and from $\varphi\vee \neg P(t)\vee \neg P(t')$ the resolvent $(\varphi\vee \neg P(t))\sigma$, provided $\sigma$ is the mgu of $P(t)$ and $P(t')$.
  For our purpose, a derivation of a clause $\varphi$ from a set of clauses $\Phi$ is a sequence of inferences where all premises are either in $\Phi$ or the resolvent of an inference occurring earlier in the sequence, and where the last resolvent is $\varphi$ itself.
  A derivation is \emph{linear} when the resolvent of one inference is always a premise of the next inference.

A general observation regarding the clauses that are relevant to this work is that,
due to the shape of clauses in $\pi(\Tmc)$, the sets $\Phi$ and $\Phi_p$ only contain clauses of the following shapes:
\begin{enumerate}[label=\textbf{I\arabic*}]
\item\label{e:left} $C_1(\skz)$,
\item\label{e:right} $\neg \negc{C_2}(\skz)$,
\item\label{e:basic} $\neg A_1(x) \vee A_2(x)$,
\item\label{e:conjunction} $\neg A_1(x) \vee \neg A_2(x) \vee A_3(x)$,
\item\label{e:negative-role} $\neg r(x,y) \vee \neg A_1(y) \vee A_2(x)$,
\item\label{e:skolem-role} $\neg A_1(x) \vee r(x,\sk(x))$, and
\item\label{e:skolem-concept} $\neg A_1(x) \vee A_2(\sk(x))$,
\end{enumerate}
where $A_1$, $A_2$ and $A_3$ are either all original literals or all duplicate literals. 
We abbreviate a ``clause of the form \textbf{I$x$}'' as an ``\textbf{I$x$}-clause'' for $x\in\{1,..,7\}$.
Observe that there is exactly one \ref{e:left}-clause and one \ref{e:right}-clause, both for the same constant $\skz$.
Moreover, for every Skolem function $\sk$ occurring in $\Phi_p$, there is exactly one pair of clauses where one is an \ref{e:skolem-role}-clause and the other an \ref{e:skolem-concept}-clause where a given $\sk\in\NS$ occurs.
We call them the clauses \emph{introducing $\sk$}.
To every Skolem function $\sk$, we associate the atomic concept $A_\sk$ that occurs positively in the \ref{e:skolem-concept}-clause introducing $\sk$.

Relying on $\Phi_p$ allows to derive all ground implicates by increasing term depth, which is possible thanks to the following result. %
\begin{lemma}
  \label{lem:no-negative-role-pip}
  It is not necessary to use \ref{e:negative-role}-clauses to derive $\PIp$ from $\Phi_p$ by resolution.
\end{lemma}
\begin{proof}
  Since $\Phi$ and $\Phi_p$ are equivalent, they have the same prime implicates, that can be derived by resolution from any of them.
  We construct a Herbrand model for $\Phi$ from all clauses in $\Phi_p$ except the \ref{e:negative-role}-clauses.
  Then, by Lemma\ \ref{lem:pipmodel}, all clauses from $\PIp$ will be included in this model and thus derivable by resolution from the restriction of $\Phi_p$ to non-\ref{e:negative-role}-clauses.

  Let $\Imc=\bigcup_i\Imc_i$ for $i\in\mathbb{N}$, such that:

  \begin{itemize}
  \item $\Imc_0=\{C_1(\skz)\}$ and,
  \item given $\Imc_j$,
    \begin{eqnarray*}
      \Imc_{j+1} &=& \Imc_j \cup \{B(t)\mid t\in (D)^{\Imc_j},\, \neg\pi(D,x)\vee B(x)\in\Phi_p\}\\
                 & &  \phantom{\Imc_j}\cup \{B(\sk(t)),r(t,\sk(t))\mid t\in (A)^{\Imc_j},\\
                 & & \qquad\quad\neg\pi(A,x)\vee B(\sk(x))\in \Phi_p,\, \neg\pi(A,x)\vee r(x,\sk(x))\in \Phi_p\}.
    \end{eqnarray*}
  \end{itemize}
  This interpretation is similar to the one used in the proof of Lemma\ \ref{lem:pipmodel}, but uses the clauses in $\Phi_p$, \ref{e:negative-role}-clauses excepted, instead of the clauses in $\Phi$.
  Thus every atom in \Imc can be derived by resolution from the clauses of $\Phi_p$ that are not \ref{e:negative-role}-clauses.

  We show that $\Imc$ is a model of $\Phi$.
  By construction, $\Imc\models C_1(\skz)$ and all \ref{e:left}-, \ref{e:basic}-, \ref{e:conjunction}-, \ref{e:skolem-role}- and \ref{e:skolem-concept}-clauses with original predicates in $\Phi$ since they also occur in $\Phi_p$.
  The clauses with duplicate predicates are also satisfied since \Imc contains no duplicates at all.
  It remains only to show that the \ref{e:negative-role}-clauses in $\Phi$ are true in \Imc. 
  By contradiction, consider that the clause $\varphi=\neg r(x,y)\vee \neg A_1(y)\vee A_2(x)$ is not satisfied by \Imc.
  Then there must exist terms $t$, $t'$ such that $r(t,t')\in\Imc$, $A_1(t')\in\Imc$ but $A_2(t)\notin\Imc$.
  By construction, $t'=\sk(t)$ for some Skolem function $\sk$.
  The only clauses with $\sk$ in $\Phi_p$ are the clauses introducing $\sk$, that we denote by $\varphi_1=\neg A_3(x)\vee r(x,\sk(x))$ and $\varphi_2=\neg A_3(x)\vee A_4(\sk(x))$ for some original atomic concept $A_4$.
  These clauses are the only possible cause for the presence of $r(t,\sk(t))$ and $A_1(\sk(t))$ in $\Imc_{j}$ for some $j\geq 1$, and thus there must be an $i<j$ s.t.\ $A_3(t)\in\Imc_i$.
  The presence of $\varphi_1$, $\varphi_2$ and $\varphi$ in $\Phi$ ensures that $\Phi\models \neg A_3(x)\vee A_2(x)$ and thus that $\neg A_3(x)\vee A_2(x)\in\Phi_p$.
  Combined with the fact that $A_3(t)\in\Imc_i$, it means that $A_2(t)\in\Imc_{i+1}\subseteq\Imc$, a contradiction.
  Thus $\Imc$ is also a model of all clauses of the form \ref{e:negative-role} in $\Phi$, so it is a model of $\Phi$ and it is possible to derive all clauses in $\PIp$ from $\Phi_p$ without using the clauses of the form \ref{e:negative-role}.
  \qedhere
\end{proof}
A direct consequence of this lemma is that, regarding derivations of $\PIp$, we only need to consider those where every inference preserves or increases the depth of terms from premises to conclusion, because the only way to decrease this depth is by using an \ref{e:negative-role}-clause.
This allows us to prove Th.\ \ref{thm:variable-bound}

\ThmVariableBound*
\begin{proof}
    By Th.\ \ref{corr:constr}, it suffices to show that all clauses in $\PIp\cup\PIn$ that contain no binary predicate
    can be derived using only inferences of clauses with at most one variable.
    Since~\ref{e:left} is the only clause containing no negative literals, any clause $\varphi\in\PIp$ must be derived using the \ref{e:left}-clause.
    Moreover \ref{e:negative-role}-clauses are the only ones in the input that would introduce a variable when resolved with a ground clause.
    By Lemma\ \ref{lem:no-negative-role-pip}, we can ignore these clauses to infer $\varphi\in\PIp$.
    Thus, any clause in $\PIp$ can be derived by inferring only ground clauses from $\Phi_p$, which is even more than what \ref{it:onevar} requires.

    For $\varphi'\in\PIn$, Lemma\ \ref{lem:no-negative-role-pip} does not apply.
    In general,  %
    any derivation from $\Phi_p$ of a clause that contains a constant involves $C_1(\skz)$ or $\neg \negc{C_2}(\skz)$ or both,
    and only \ref{e:negative-role}-clauses would introduce a variable into such a derivation.
    Let $\varphi$ be the first clause with a variable that occurs as a resolvent in the derivation of $\varphi'$ from $\Phi_p$, and let $\varphi'$ be without binary predicates, since it must be usable to build a constructible hypothesis following Th.\ \ref{corr:constr}.
    The premises of the inference producing $\varphi$ are a ground clause and an \ref{e:negative-role}-clause, $\neg r(x,y)\vee \neg A_1(y)\vee A_2(x)$.
    We show that any occurrence of a variable in $\varphi$ can be immediately eliminated by another inference, creating a new derivation for $\varphi'$.
    Depending on the literal resolved upon in the \ref{e:negative-role}-clause to obtain $\varphi$ several cases occur.
    \begin{itemize}
    \item This literal cannot be $\neg r(x,y)$, or $\varphi$ would be ground because both $x$ and $y$ would be unified with ground terms.  %
    \item If the literal resolved upon is $A_2(x)$, then $y$ occurs in $\varphi$ as its only variable, in the literals $\neg r(t,y)$ for some ground $t$ and $\neg A_1(y)$.
      The literal $\neg r(t,y)$ is eliminated later in the derivation since $\varphi'$ contains no binary predicate.
      All positive occurrences of $r$ that can be derived are of the form $r(t',\sk(t'))$ for some $\sk$,
      because in the $\Phi_p$, the only positive occurrences of $r$ are found in \ref{e:skolem-role}-clauses. %
      Thus, to obtain a clause in $\PIn$ without roles, we need to eventually unify the variable $y$ with a ground term of the form
      $\sk(t)$ for some $\sk$.
      Since $\Phi_p$ is Horn, we can rearrange any derivation from $\varphi$ to $\varphi'$ so that we first resolve upon $\neg r(t,y)$
      in $\varphi$ with the suitable \ref{e:skolem-role}-clause, i.e., the one introducing the appropriate $\sk$.
      As a result, we obtain another ground clause with no variables, before any further inference is performed if needed.

    \item If the literal resolved upon is  $\neg A_1(y)$, then $x$ occurs in $\varphi$ in the literals $A_2(x)$ and $\neg r(x,t)$, where $t$ is ground.
      The argument unfolds as in the previous case, with the nuance that the considered $\sk$ function is the one s.t.\ $t=\sk(t')$ for some $t'$.
    \end{itemize}
    It follows that a derivation of $\varphi'$ that does not respect \ref{it:onevar} can always be rearranged to eliminate occurrences of variables (and binary literals) as soon as they occur, before the next variable is introduced.
    The rearranged derivation respects \ref{it:onevar}.
    \qed

\end{proof}

  The proof of Th.\ \ref{the:bound} is based on a structure called a \emph{solution tree}, that resembles a description trees, but instead collects in its (multiple) labels information on all the clauses that helped derive the prime implicates needed to build a constructible hypothesis.
A solution tree for a hypothesis $\Hmc$ is defined as
tuple $\tup{\forest,\plab,\nlab}$, which is a tree-shaped labeled graph $\forest=\tup{V,E,\sklab}$ together with two additional labeling functions $\plab$ and $\nlab$.
The leaves $v_1$, $\ldots$, $v_n$ of a solution tree are such that
\begin{itemize}
\item $\neg\nlab(\sklab(v_1))\vee \ldots\vee \nlab(\sklab(v_n))\in\PIn$ up to the repeated occurrence of some literals from different nodes,
\item $\plab(v_i)\in\PIp$ is not empty for all $i\in\{1,\ldots,n\}$, and
\item these are the matching prime implicates used to construct the hypothesis following Th.\ \ref{corr:constr}.%
\end{itemize}
Theorem\ \ref{the:bound} is proved by showing how to reduce a solution tree to a smaller one when the associated constructible hypothesis is not subset-minimal.

Towards proving Th.\ \ref{the:bound}, we step-wise introduce the details of these notions,
and, along the way, we prove their relevant properties. We start with the fragment called the
Skolem tree.

\begin{definition}[Skolem Tree]
A \emph{Skolem tree} is a labeled tree $\forest=\tup{V,E,\sklab}$ where $\sklab$ assigns a Skolem term to every $v\in V$ s.t.\ 
$\sklab(v_0)=\skz$ for the root $v_0\in V$, and for every $\tup{v,v'}\in E$,\footnote{When the role $r$ labeling an edge $vrw$ is irrelevant, we fall back to representing this edge as the pair of nodes $(v,w)$.}
either $\sklab(v)=\sklab(v')$ or $\sklab(v')=\sk(\sklab(v))$ for some Skolem term $\sk$.
\end{definition}
Chains and antichains for Skolem trees are defined as usual, that is, a \emph{chain} is a set of nodes
that occur together on a path, and an \emph{antichain} is a set of nodes such that no node is an
ancestor of another node. \emph{Maximal} chains/antichains are chains/antichains that are maximal w.r.t.\ the subset relation. Given two nodes $v,v'\in V$, we call $v$ an \emph{ancestor} of $v'$ iff there is a path
leading from $v$ to $v'$. This implicitly implies that every node is an ancestor of itself.

\begin{definition}[Positive Labeling]
  \label{def:poslab}
The \emph{positive labeling for $\forest$} is defined as the function
$\plab:V\rightarrow2^{\NC}$ s.t.
for every $v\in V$, $\plab(v)=\{A\mid A(\sklab(v))\in \PIp\}$.
\end{definition}

\begin{lemma}\label{lem:ancestors-left}
  Let $v_1,v_2\in V$ be such that $v_1$ is an ancestor of $v_2$, and $\sklab(v_1)$ be of the form $\sk(t)$ for some Skolem function \sk.
  Then,
  $A_\sk(\sklab(v_1))\in\PIp$ and
  for every $A\in \plab(v_2)$,
  there is a derivation of $A(\sklab(v_2))$ from $A_\sk(\sklab(v_1))$ and the \ref{e:basic}-, \ref{e:conjunction}- and \ref{e:skolem-concept}-clauses in $\Phi_p$.
\end{lemma}
\begin{proof}
  We first show that for every $A(\sk(t))\in\PIp$, also $A_\sk(\sk(t))\in\PIp$,
  and that $A(\sk(t))$ can be derived from $A_\sk(\sk(t))$ and \ref{e:basic}- and \ref{e:conjunction}-clauses of $\Phi_p$.
  Afterward, we prove that if $\plab(v_2)$ is not empty, then $\plab(v_1)$ is also not empty for any ancestor $v_1$ of $v_2$. 
  The lemma then follows by induction.

  We have established in the proof of Th.\ \ref{thm:variable-bound} that $A(\sk(t))$ can be derived from $\Phi_p$ such that every resolvent of the derivation is ground.
  Moreover, by Lemma\ \ref{lem:no-negative-role-pip}, \ref{e:negative-role}-clauses do not have to be involved in the derivation.
  \ref{e:skolem-role}-clauses can then also be ignored because they introduce positive occurrences of binary literals and only \ref{e:negative-role}-clauses can be used to resolve upon them.
  We can also rule out the \ref{e:right}-clause $\neg \negc{C_2}(\skz)$ because it is a duplicate, that can only derive negative duplicate clauses in $\Phi_p$ if \ref{e:negative-role}- and \ref{e:skolem-role}-clauses are not used, because $\Phi_p$ is Horn.
  This leaves $C_1(\skz)$ and the \ref{e:basic}-, \ref{e:conjunction}- and \ref{e:skolem-concept}-clauses as the ones that are used to derive $A(\sk(t))$ from $\Phi_p$.
  Moreover, since $\Phi_p$ is Horn, linear resolution can be used to derive $A(\sk(t))$, thus inferences between two resolvents are not necessary\ \cite{AndersonBledsoe70}. 
  Inferences in such a derivation can only preserve the Skolem term occurring in the ground premise when resolving with an \ref{e:basic}- or an \ref{e:conjunction}-clause, and increase the depth of the term in the resolvent when resolving with an \ref{e:skolem-concept}-clause.
  Thus $A(\sk(t))$ can only be derived after $A_\sk(\sk(t))$ has been introduced by the \ref{e:skolem-concept}-clause introducing $\sk$, and from $A_\sk(\sk(t))$ only \ref{e:basic}- or \ref{e:conjunction}-clauses can be used to derive $A(\sk(t))$.
  Let us consider the (linear) derivation of $A(\sk(t))$ and remove from it all the inferences on \ref{e:basic}- and  \ref{e:conjunction}-clauses upon literals where $\sk(t)$ occurs.
  The only literals that remain in the derived clause are copies of $A_\sk(\sk(t))$ because it is the only literal with the term $\sk(t)$ that can be derived, and because all other literals are resolved upon in parts of the derivation that have not been removed.
  Thus, it is enough to append a few factorization inferences at the end of this derivation, if at all needed, to derive $A_\sk(\sk(t))$.
  Hence $A_\sk(\sk(t))\in\PIp$.
  The inferences from the removed parts of the derivation of $A(\sk(t))$, introducing only \ref{e:basic}- and \ref{e:conjunction}-clauses can be used together with $A_\sk(\sk(t))$ to construct a derivation of $A(\sk(t))$.

  To prove that ancestors of $v_2$ have a non-empty positive label if $v_2$ has a non-empty positive label, let us consider the case when $v_1$ is the direct parent of $v_2$ (the case when $v_1=v_2$ is trivial). 
  Let us consider once again the part of the linear derivation of $A(\sk(t))$ from which we created the derivation of $A_\sk(\sk(t))$.
  If we remove from it the inference(s) introducing $\sk$, the clause that is derived must contain only $A'(t)$ literals, where $\neg A'(t)\vee A_\sk(\sk(t))$ is the \ref{e:skolem-concept} clause that introduced $\sk$.
  These $A'(t)$ literals can be factorized following this derivation to obtain a derivation of $A'(t)$ from $\Phi_p$.
  Thus $v_1$ has a non-empty positive label and this result also holds for any ancestor of $v_2$ by induction.

  Finally, we have proven that for $v_1$ s.t.\ $\sklab(v_1)=\sk'(t')$, there is a derivation of any $A'(\sk'(t'))\in\plab(v_1)$ from $A_{\sk'}(\sk'(t'))$ and the \ref{e:basic}- and \ref{e:conjunction}-clauses in $\Phi_p$.
  Moreover, in the previous paragraph, assuming $v_1$ is the parent of $v_2$, we have seen that there is at least some $A'(\sk'(t'))\in\plab(v_1)$ from which $A_{\sk}(\sk(\sk'(t')))$ can be inferred by using the \ref{e:skolem-concept}-clause introducing $\sk$ where $\sk(\sk'(t'))=\sklab(v_2)$.
  Thus $A(\sklab(v_2))$ can be derived from $A_{\sk'}(\sklab(v_1))$ and the \ref{e:basic}-, \ref{e:conjunction}- and \ref{e:skolem-concept}-clauses in $\Phi_p$, and this result can be extended to any ancestor of $v_2$ (except the root $v_0$, for which $C_1(\skz)$ could be used instead of $A_{\sk'}(\sklab(v_0))$, that does not exist).
  \qed
\end{proof}
The following lemma is central to the proof of the theorem.
It bounds the range of the positive labeling by ensuring that each path from the root to a leaf of the tree that is longer than the size of $\NS$ contains two nodes with the same positive labeling.
\begin{lemma}
  \label{lem:unique-left}
 For any two nodes $v_1,v_2\in V$ s.t. $\sklab(v_1)=\sk(t_1)$ and $\sklab(v_2)=\sk(t_2)$ for the same Skolem function $\sk$ and some terms $t_1$ and $t_2$, if $\plab(v_1)$ and $\plab(v_2)$ are not empty, then
 $\plab(v_1)=\plab(v_2)$.
\end{lemma}
\begin{proof}
  By Lemma~\ref{lem:ancestors-left}, if $\plab(v_1)$ and $\plab(v_2)$ are not empty, then $A_\sk(\sklab(v_1))\in\plab(v_1)$ and $A_\sk(\sklab(v_2))\in\plab(v_2)$.
  Moreover, for every $A\in \plab(v_1)$, $A(\sklab(v_1))$ can be derived using only $A_\sk(\sklab(v_1))$ and \ref{e:basic}- and \ref{e:conjunction}-clauses in $\Phi_p$.
  The \ref{e:skolem-concept}-clauses do not intervene because they can only increase the depth of terms.

  Any such derivation can be mirrored from $A_\sk(\sklab(v_2))$ and the \ref{e:basic}- and \ref{e:conjunction}-clauses in $\Phi_p$ to derive $A(\sklab(v_2))$, thus $\plab(v_1)\subseteq \plab(v_2)$.
  The reverse inclusion $\plab(v_2)\subseteq \plab(v_1)$ is proved similarly.
\end{proof}

\begin{lemma}
  \label{lem:role-rep}
  If $r(t,\sk(t))\in\PIp$ and $\sk(t)$ has a subterm of the form $\sk(t')$, then $r(t',\sk(t'))\in\PIp$.
\end{lemma}
\begin{proof}
  If $r(t,\sk(t))\in\PIp$, then $r$ occurs in the \ref{e:skolem-role}-clause introducing $\sk$, thus any (linear) derivation of $r(t,\sk(t))$ can be transformed into a derivation of $A_\sk(\sk(t))$ by replacing this \ref{e:skolem-role}-clause with the \ref{e:skolem-concept}-clause introducing $\sk$ when resolving upon $A_1(t)$ (remember it is possible to ensure that the other premise is ground).
  Thus $A_\sk(\sk(t))\in\PIp$.

  In a Skolem tree with a node $v_2$ s.t.\ $\sklab(v_2)=\sk(t)$ for the $t$ and $\sk$ from the previous paragraph, there must be an ancestor $v_1$ of $v_2$ s.t.\ $\sklab(v_1)=\sk(t')$ since it is a subterm of $\sk(t)$.
  Thus by Lemma\ \ref{lem:unique-left}, $A_\sk(\sk(t'))\in\PIp$.
  Moreover, a (linear) derivation of $A_\sk(\sk(t'))$ from $\Phi_p$ can be turned into a derivation of $r(t',\sk(t'))$ by applying the reverse transformation as in the previous paragraph, thus $r(t',\sk(t'))\in\PIp$.
  \qed
\end{proof}

The last remaining ingredient to obtain the solution tree is the negative labeling.

\begin{definition}[Negative Labeling]
  \label{def:neglab}
A function $\nlab:V\rightarrow\{\negc{A}\mid A\in\NC\}$ is called \emph{negative labeling for $\forest$} and $\Bmc$ iff:
\begin{itemize}
    \item The root $v_0$ of \forest is s.t.\ $\nlab(v_0)=C_2$ and $\sklab(v_0)=\skz$. %
    \item
      For all nodes $v$ in \forest that are not leaves, there is a derivation of $\neg\negc{B_1}(\sklab(v_1))\vee\ldots\vee \neg\negc{B_n}(\sklab(v_n))$ from $\neg\negc{B}(\sklab(v))$, the set $\{r(t,\sk(t))\in\PIp\mid t\in\gterms{\skz}, \sk\in\NS\}$ and the non-ground clauses in $\Phi_p$, where $B=\nlab(v)$, $B_i=\nlab(v_i)$ and $\{v_1,\ldots,v_n\}\subseteq V$ is the set of children of $v$.
\end{itemize}
\end{definition}
Given a Skolem tree \forest with leaves $v_1$,\ldots, $v_n$, and a negative labeling $\nlab$, we denote by \impfrmtr the clause $\neg\negc{B_1}(\sk(v_1))\vee\ldots\vee \neg\negc{B_n}(\sk(v_n))$ where $B_i=\nlab(v_i)$ for all $i\in\{1,\ldots,n\}$.
\begin{lemma}
  \label{lem:derivable}
If a Skolem tree $\forest$ has a negative labeling $\nlab$, then every maximal antichain $\{v_1,\ldots,v_m\}$ in \forest corresponds to a clause $\neg \negc{B_1}(t_1)\vee\ldots\vee\neg \negc{B_m}(t_m)$ that can be derived from $\Phi_p$.
\end{lemma}
This follows directly by induction starting from the root and it holds in particular for \impfrmtr.
Note that any such clause $\neg \negc{B_1}(t_1)\vee\ldots\vee\neg \negc{B_m}(t_m)$ is a ground negative implicate, not necessarily prime.
Conversely, it is not hard to see that for any negative ground implicate of $\Phi_p$ 
of the form $\neg \negc{B_1(t_1)}\vee\ldots\vee\neg \negc{B_n(t_n)}$, we can construct a Skolem tree and a negative labeling
s.t.\ \impfrmtr is the negative ground implicate, up to the repetition of literals.
This formulation allows the same $B_i$ to be represented by several leaves, which is necessary because the tree captures resolution inferences in the derivation but not factorization inferences.

This Skolem tree can be constructed together with the negative labeling $\nlab$ by following the derivation from $\neg \negc{C_2}(\skz)$ to $\neg \negc{B_1}(t_1)\vee\ldots\vee\neg \negc{B_n}(t_n)$ in $\Phi_p$. Specifically, we note that:
\begin{itemize}
\item Such a derivation must exist, because $\Phi_p$ is Horn and $\neg \negc{C_2}(\skz)$ is the only negative clause and thus every derivation of a negative ground clause must use this clause.
 \item Moreover, starting from $\Phi_p\cup\PIp$ all resolvents in a derivation of a negative ground clause can be negative clauses and only $\{r(t,\sk(t))\in\PIp\mid \sk\in\NS, t\in\gterms{\skz}\}$ is needed because the other positive prime implicates contain original literals that are not needed to derive a clause with only duplicate literals. %
   Such derivations can be linear.
 \item Following the argument used in the proof for Th.\ \ref{thm:variable-bound}, we can rearrange any linear derivation of a ground negative clause so that variables and binary literals are eliminated as soon as they are introduced.
   This step is done by resolving first a ground clause $\varphi$ and an \ref{e:negative-role}-clause $\neg r(x,y)\vee\neg A_1(y)\vee A_2(x)$, and then resolving an \ref{e:skolem-role}-clause $\neg A_1'(x')\vee r(x',\sk(x'))$ and the resolvent of the previous step.
   These two steps amount to replacing the literal $\neg A_2(t)$ in $\varphi$ by $\neg A_1'(t)\vee\neg A_1(\sk(t))$.
\end{itemize}
We thus obtain the following lemma.
\begin{lemma}\label{lem:right-labeling}
  For any $\Bmc$ s.t. $\bigvee_{B(t)\in\Bmc}\neg \negc{B}(t)\in\PIn$, we can construct a Skolem tree $\forest$ that has a negative labeling $\nlab$ s.t.\ for every $B(t)\in\Bmc$, $\sklab(v)=t$ and $\nlab(v)=B$ for a leaf $v$ of $\forest$, and %
for all leaves $v$ in \forest, $\nlab(v)(\sklab(v))\in\Bmc$.
\end{lemma}
Combining all three labels, we obtain a solution tree.
By reading off the three labels on the leaves of a solution tree, we obtain a connection-minimal hypothesis after
Th.\ \ref{corr:constr}, when the positive labels of all leaves are not empty.
\begin{definition}
  A \emph{solution tree} is a tuple $\tup{\forest,\plab,\nlab}$ where
  \begin{itemize}
  \item $\forest = \tup{V,E,\sklab}$ is a Skolem tree for the terms from which some constructible hypothesis \Hmc is defined,
  \item $\plab$ is a positive labeling for \forest such that all nodes in \forest have a non-empty positive label, and
  \item $\nlab$ is a negative labeling for \forest s.t.\ $\impfrmtr\in\PIn$ modulo the repetition of literals.
  \end{itemize}
  The \emph{solution} of the tree is a TBox equivalent to \Hmc, which is the set $\{\bigsqcap \plab(v)\sqsubseteq \nlab(v)\mid v\text{ is a leaf of }\forest\}$.
\end{definition}
Note that the CIs in the solution of the tree use only one atomic concept on the right-hand side, while the equivalent \Hmc may contain a conjunction.
Moreover, removing all tautologies from this solution results in a packed connection-minimal hypothesis.

\begin{lemma}\label{lem:pumping}
  Let $\tup{\forest,\plab,\nlab}$ be a solution tree, where $\forest=(V,E,\sklab)$ and $v_1$, $v_2\in V$ be such that $v_1$ is an ancestor of $v_2$, $\sklab(v_1)=\sk(t)$ and $\sklab(v_2)=\sk(t')$ for some $\sk$, $t$ and $t'$, %
  and $\nlab(v_1)=\nlab(v_2)$. Let $\forest'=\tup{V',E',\sklab'}$ the result of replacing in $\forest$ the subtree under $v_1$ by the subtree under $v_2$, adapting the Skolem labeling $\sklab$ to $\sklab'$ appropriately, and let $\plab'$ and $\nlab'$ be $\plab$ and $\nlab$ restricted to $V'$. Then, $(\forest',\plab',\nlab')$ is also a solution tree.
\end{lemma}
\begin{proof}
  We have to show that all primed labelings are valid labelings, that the positive one is not empty for any node in \forest' and that $\impfrmtrp\in\PIn$ modulo the repetition of literals. %

  The adaptation of $\sklab$ to create $\sklab'$ consists in replacing the subterm $\sk(t')$, in every term $\sklab(v)$ for $v$ descending from $v_2$ in \forest by $\sk(t)$ in \forest'.
  That way, $\sklab'$ is also a Skolem labeling.
  By Def.\ \ref{def:poslab} and Lemma\ \ref{lem:unique-left}, $\plab'$ is a positive labeling for \forest' and none of its labels are empty because none of the labels of $\plab$ are empty.
  By Def.\ \ref{def:neglab}, all properties needed to ensure $\nlab'$ is a negative labeling for \forest' are trivially verified for the nodes outside of the descendants of $v_1$ because they are the same as in \forest.

  Let $w_1$, \dots, $w_n$ be the children of $v_1$ in \forest'.
  We show that $\varphi=\neg\negc{B_1}(\sklab(w_1))\vee\ldots\vee\neg\negc{B_n}(\sklab(w_n))$ where $B_i=\nlab(w_i)$ for $i\in\{1,\ldots,n\}$ can be derived from the non-ground clauses in $\Phi_p$, the set $\{r(t',\sk(t'))\in\PIp\mid \sk\in\NS,t'\in\gterms{\skz}\}$ and $\neg \negc{B}(\sklab(v_1))$, where $B=\nlab(v_1)$.
  In \forest, the $w_i$ nodes are the children of $v_2$, thus $\varphi'=\neg\negc{B_1}(\sklab(v_2))\vee\ldots\vee\neg\negc{B_n}(\sklab(v_2))$ can be derived from the non-ground clauses in $\Phi_p$, the set $\{r(t',\sk(t'))\in\PIp\mid \sk\in\NS,t'\in\gterms{\skz}\}$ and $\neg \negc{B}(\sklab(v_2))$.
  We can write $\sklab(v_2)$ as $g(\sklab(v_1))$, where $g$ is a composition of Skolem functions, and every $\sklab(w_i)$ can be written as $\sk_i(g(\sklab(v_1)))$ for some $\sk_i\in\NS$.
  The derivation of $\varphi'$ can thus be transformed into a derivation of $\varphi$ by replacing $B(t')$ by $B(t)$ everywhere it is used, and replacing any $r(\sklab(v_2), \sk_i(\sklab(v_2)))\in\PIp$ used by $r(\sklab(v_1),\sk_i(\sklab(v_1)))$ because the latter also belongs to $\PIp$ by Lemma\ \ref{lem:role-rep}. 
  The same argument applies to any other descendant $v'$ of $v_1$ in \forest' so that the second point of Def.\ \ref{def:neglab} holds for $\nlab'$.
  Thus $\nlab'$ is a negative labeling for \forest'.

  If $\nlab'$ is such that \impfrmtrp is not in $\PIn$ modulo the repetition of literals, then there exists a solution tree \forest'' for a strict subclause of \impfrmtrp modulo the repetition of literals, i.e., where one literal of \impfrmtrp does not appear at all, that is an implicate of $\Phi_p$ and it is possible to apply the transformation from \forest to \forest' backward from \forest''.
  This would produce a solution forest for a strict subclause of \impfrmtr modulo the repetition of literals derivable from $\Phi_p$ by Lemma\ \ref{lem:derivable}, which is impossible since $\impfrmtr\in\PIn$ modulo the repetition of literals. 
  Thus $\impfrmtrp\in\PIn$ modulo the repetition of literals.
\end{proof}

\begin{definition}
    A solution tree $\forest$ for a hypothesis $\Hmc$ is \emph{minimal} if there exists no solution tree $\forest'$ for a hypothesis $\Hmc'\subseteq \Hmc$ s.t.\ $\forest'$ uses strictly less nodes.
\end{definition}

\begin{lemma}\label{lem:no-repetition}
  Let $\tup{\forest,\plab,\nlab}$ be a minimal solution tree where $\forest=(V,E,\sklab)$.
  Let $v,v'\in V$ be nodes such
    that $v'$ is an ancestor of $v$. Then, either $\sklab(v)$ and $\sklab(v')$ are not headed by the same Skolem term or $\nlab(v)\neq \nlab(v')$. %
\end{lemma}
\begin{proof}
  Let $\tup{\forest,\plab,\nlab}$ be a minimal solution tree for the hypothesis $\Hmc$ where $\forest=(V,E,\sklab)$ such that $v'$ is an ancestor of $v$, $\nlab(v)=\nlab(v')$ and $\sklab(v)=\sk(t)$ and $\sklab(v')=\sk(t')$ for some $\sk$, $t$ and $t'$. By applying Lemma~\ref{lem:pumping}, we obtain a solution tree $\forest'$ with less nodes than $\forest$ and for a solution $\Hmc'$ s.t. $\Hmc'\subseteq \Hmc$. Consequently, $\forest$ cannot be minimal.\qed
\end{proof}

\begin{corollary}\label{cor:depth-bound}
  Let $=\tup{\forest, \plab,\nlab}$ be a minimal solution tree.
  Then, the depth of \forest is bounded by $n\times m$, where $n$ is the number of Skolem functions in $\Phi$ introduced for the transformation of $\Tmc$, and $m$ is the number of atomic concepts in $\Phi$.
\end{corollary}
\begin{proof}
  Let $\forest$ be a minimal solution tree.
  By Lemma~\ref{lem:no-repetition}, on every path in $\forest$, there are no two nodes $v$, $v'$ such that $\nlab(v)=\nlab(v')$, $\sklab(v)=\sk(t)$ and $\sklab(v')=\sk(t')$ for some $\sk$, $t$ and $t'$.
Thus in a path, for every Skolem function $\sk$, there can be at most $m$ nodes on a path, each with a different negative label.
However, the Skolem functions introduced during the translation of $\negc{\Tmc}$ are never needed since they do not occur in $\PIp$.
The range of $\nlab$ is bounded by the number of atomic concepts in $\Phi$, that we denote $m$.
We additionally denote by $n$ the number of Skolem functions in $\Phi$ introduced by the translation of \Tmc, and thus, every path in \forest can have a length of at most $n\times m$.\qed
\end{proof}

Theorem~\ref{the:bound} is a direct consequence of Corollary~\ref{cor:depth-bound}, because for a subset-minimal constructible hypothesis \Hmc, there is no constructible hypothesis \Hmc' s.t.\ $\Hmc'\subsetneq\Hmc$.

\section{Locality-based Modules}
\label{app:modules}
Realistic ontologies easily get too large to be processed by SPASS in reasonable
time for the abduction task. We therefore use module extraction to obtain
a relevant subset of the ontology before translating the abduction problem.
For a signature $\Sigma$ and an ontology $\Tmc$, a module $\Mmc$ of $\Tmc$ for
$\Sigma$ is a subset of $\Tmc$ that preserves all entailments of closed second-order
formulas that only use predicates from $\Sigma$.
However, for our particular reasoning task, the signature $\Sigma$ from the
abduction problem $\langle \Tmc,\Sigma, C_1\isa C_2 \rangle$ is not known in
advance, so that we have to be a bit more careful when extracting the module.

Specifically, we need to ensure that for the observation $C_1\sqsubseteq C_2$,
the module $\Mmc$ preserves all subsumers of  $C_1$ and all subsumees of $C_2$,
as these are the backbone of connection minimality
(see Def.\ \ref{def:abd}).
This can be done using special \emph{locality-based modules}, as presented by
Grau et al.\ \cite{DBLP:journals/jair/GrauHKS08}.

\begin{definition}
  A CI $\alpha$ is $\emptyset$-local (reps.\ $\Delta$-local) for a signature $\Sigma$ if every interpretation $\Imc$ s.t.\ $X^\Imc=\emptyset$ (resp.\ $X^\Imc=\Delta$) for all $X\in(\NC\cup\NR)\setminus\Sigma$ satisfies $\Imc\models\alpha$.
\end{definition}

\begin{definition}
  The \emph{$\emptyset$-module} (resp.\ \emph{$\Delta$-module}) of $\Tmc$ for
  $\Sigma$ is the smallest subset $\Mmc\subseteq\Tmc$ s.t.\ every axiom in
  $\Tmc\setminus\Mmc$ is \emph{$\emptyset$-local} (resp.\ $\Delta$-local) for
  $\Sigma\cup\sig(\Mmc)$, where $\sig(\Mmc)$ denotes the restriction of the signature to the symbols occurring in the TBox \Mmc.
\end{definition}
This means the axioms outside of the $\emptyset$-module $\Mmc$ do not contribute to non-trivial entailments using the signature $\Sigma\cup\sig(\Mmc)$.

A fast approximation of $\emptyset$-modules are $\bot$-modules.
The exact definition of $\bot$-modules is given in \cite{DBLP:journals/jair/GrauHKS08}, but not needed for the
following.
It suffices to know that if $\Mmc'$ is the $\bot$-module of $\Tmc$ for $\Sigma$, and $\Mmc$ a $\emptyset$-module of
$\Tmc$ for $\Sigma$,
then $\Mmc\subseteq\Mmc'$. In the same way, $\top$-modules approximate $\Delta$-modules. The relevant property for us
is the following.

\begin{lemma}\label{lem:bot-modules}
  For a concept $C_1$ and a TBox \Tmc, the $\bot$-module $\Mmc$ of $\Tmc$ for $\sig(C_1)$ satisfies $\Tmc\models C_1\isa D$ iff $\Mmc\models C_1\isa D$ for all concepts $D$.
\end{lemma}
\paragraph{Proof sketch.}
  Thanks to the relation between the $\bot$-module and the $\emptyset$-module, it suffices to prove that the $\emptyset$-module $\Mmc$' for $\sig(C_1)$ satisfies the property to obtain the same result for the $\bot$-module $\Mmc$.
  We first observe that every non-tautological axiom $C\sqsubseteq C'\in\Tmc$ s.t.\ $\sig(C)\subseteq\sig(\Mmc)$
  occurs in $\Mmc$.
  Otherwise, we would have $\sig(C')\not\subseteq(\Mmc)$, and $C'^\Imc=\emptyset$
  for an interpretation $\emptyset$-local for $\sig(\Mmc)$, while $C^\Imc\neq\emptyset$, and thus
  $\Imc\not\models C\sqsubseteq D$.
  Moreover, in \EL, all subsumers of $C_1$ can be generated by unfolding, i.e.,
  by iteratively replacing sub-concepts $C$ in $C_1$ by concepts $C'$ s.t.\ $C\sqsubseteq C'\in\Tmc$.
  By using our first observation, we obtain that any axiom $C\sqsubseteq C'\in\Tmc$ that could be involved by such
  an unfolding operation must be included in $\Mmc$.
  It follows then that $\Mmc\models C_1\sqsubseteq D$ iff $\Tmc\models C_1\sqsubseteq D$ for all concepts~$D$.
  \qedhere
  \vspace{1em}

  \noindent

\begin{lemma}\label{lem:top-modules}
  For a concept $C_2$, the $\top$-module $\Mmc$ of $\Tmc$ for $\sig(C_2)$ satisfies
  $\Tmc\models D\isa C_2$ iff $\Mmc\models D\isa C_2$ for all concepts $D$.
\end{lemma}
\paragraph{Proof sketch.}
  Can be shown in the same way as Lemma~\ref{lem:bot-modules}. \qedhere
  \vspace{1em}

  \noindent
It follows from Lemma~\ref{lem:bot-modules} and~\ref{lem:top-modules}, as well as from the definition of
connection minimality (Def.\ \ref{def:abd}), that by replacing $\Tmc$ in the
abduction problem
by the union of the $\bot$-module for $\sig(C_1)$ and the $\top$-module for $\sig(C_2)$, we do not loose solutions
of the original problem.

\end{document}